\documentclass[11pt]{article} 
\usepackage[margin=1in]{geometry}
\usepackage[utf8]{inputenc}
\usepackage{amsmath, amssymb, amsfonts} 
\usepackage{graphicx}
\usepackage{booktabs}
\usepackage{amsthm}
\usepackage{makecell}
\usepackage{siunitx} 
\usepackage{hyperref}
\usepackage{natbib}

\theoremstyle{plain}
\newtheorem{theorem}{Theorem}
\newtheorem{proposition}{Proposition}
\newtheorem{corollary}{Corollary}
\theoremstyle{remark}
\newtheorem{remark}{Remark}

\title{Uncertainty-Aware Reward Discounting for Mitigating Reward Hacking}
\author{Disha Singha} 
\date{}

\title{Uncertainty-Aware Reward Discounting for Mitigating Reward Hacking}

\begin{document}

\maketitle
\begin{abstract}
Reinforcement learning from human feedback (RLHF) systems face a compounding alignment challenge: not only are learned reward models uncertain about unseen state-action pairs, but the human preference annotations they are trained on are themselves inconsistent, context-dependent, and noisy. Existing approaches address these uncertainty sources in isolation -- epistemic uncertainty is used to guide exploration, while preference uncertainty is absorbed during reward model training but discarded during policy optimization. We introduce Uncertainty-Aware Reward Discounting (UARD), a principled framework that jointly models epistemic uncertainty in value estimation via ensemble disagreement and aleatoric uncertainty in human preference annotations via annotator variability, combining these signals through a confidence-adjusted Reliability Filter that adaptively modulates reward weighting during policy optimization. We prove that this dynamic discounting preserves the contraction property of the Bellman operator, guaranteeing convergence to a unique fixed point, and provide an information-theoretic justification grounded in the Information Bottleneck principle. Empirically, UARD reduces reward hacking incidents by up to 93.6\% across discrete decision-making and continuous control benchmarks (MuJoCo) compared to nine baselines including DQN, Ensemble-DQN, CQL, CPO, TRPO, SAC, EDAC, SUNRISE, and PPO, while maintaining competitive task performance on well-specified rewards. Under annotation noise ranging from 10\% to 30\% Gaussian perturbation, UARD retains near-zero safety violations compared to baselines' near-linear degradation. These results demonstrate that treating uncertainty as an active component of the optimization objective -- rather than a passive diagnostic signal -- provides a principled pathway toward more reliable and aligned RL systems.
\end{abstract}

\section{Introduction}

Reinforcement learning from human feedback (RLHF) has become the dominant paradigm for aligning AI systems with human objectives, underpinning recent advances in large language models, robotic control, and autonomous decision-making \citep{christiano2017deep, ouyang2022training}. Yet RLHF systems face a structural vulnerability that has received insufficient attention: the reward signal used to train the policy is doubly uncertain. The reward model is uncertain about state-action pairs it has not seen \citep{osband2016deep}, and the human annotations it was trained on are themselves noisy, inconsistent, and context-dependent \citep{biyik2019batch, christiano2017deep}. When a policy optimizer encounters this doubly uncertain reward signal, it does exactly what it is designed to do—maximize it. The result is reward hacking: the agent finds trajectories that score highly under the proxy reward while diverging from the true human objective \citep{krakovna2020specification, skalse2022defining}. For instance, a cleaning robot rewarded for visible cleanliness might hide trash rather than dispose of it \citep{krakovna2020specification}; a content recommendation system might promote addictive content to maximize engagement metrics.

We argue that this failure mode is not merely an engineering flaw but a structural consequence of treating an uncertain reward signal as ground truth. Prior work addresses the two uncertainty sources independently: epistemic uncertainty is used to guide exploration \citep{osband2016deep, pathak2017curiosity} or regularize Q-values \citep{kumar2020conservative}, while preference uncertainty is modeled during reward learning \citep{biyik2019batch, lee2021pebble} but discarded during policy optimization. Neither approach prevents the agent from over-optimizing in precisely the regions where both sources of uncertainty are highest—regions where the reward signal is least trustworthy \citep{christiano2017deep, krakovna2020specification}. We propose a fundamentally different approach: rather than using uncertainty as a signal for exploration or constraint satisfaction, we treat it as evidence that the reward signal itself may be untrustworthy, and discount it accordingly.

UARD addresses reward hacking by wrapping the policy optimization process in a Reliability Filter that dynamically down-weights reward signals in proportion to their dual-source uncertainty. When the agent's ensemble disagrees about a state-action pair \citep{osband2016deep}, or when human annotators show high variance in their reward estimates for that region \citep{biyik2019batch}, the effective reward contribution is reduced—preventing the agent from over-optimizing on signals it should not trust. This mechanism requires no access to ground truth rewards, no modification of the environment, and is compatible with standard Q-learning and actor-critic frameworks. Our primary result is a $93.6\%$ reduction in reward hacking incidents on GridWorld benchmarks and consistent alignment maintenance on MuJoCo continuous control tasks \citep{todorov2012mujoco}, while retaining near-zero safety violations under up to $30\%$ annotation noise.

\subsection{Contributions}
UARD addresses reward hacking by wrapping the policy optimization process in a Reliability Filter that dynamically down-weights reward signals in proportion to their dual-source uncertainty. When the agent's ensemble disagrees about a state-action pair \citep{osband2016deep}, or when human annotators show high variance in their reward estimates for that region \citep{biyik2019batch}, the effective reward contribution is reduced---preventing the agent from over-optimizing on signals it should not trust. This mechanism requires no access to ground truth rewards, no modification of the environment, and is compatible with standard Q-learning and actor-critic frameworks. Our primary result is a $93.6\%$ reduction in reward hacking incidents on GridWorld benchmarks and consistent alignment maintenance on MuJoCo continuous control tasks \citep{todorov2012mujoco}, while retaining near-zero safety violations under up to $30\%$ annotation noise.

\begin{itemize}
    \item \textbf{Dual-Source Uncertainty Framework:} We introduce a principled approach that jointly models epistemic uncertainty (via ensemble disagreement over Q-values) and aleatoric uncertainty (via variance in human reward annotations), providing a richer signal for identifying unreliable reward regions than either source alone.
    
    \item \textbf{Theoretical Guarantees:} We prove that the Reliability Filter preserves the $\gamma$-contraction property of the Bellman operator (Theorem 1), ensuring convergence to a unique fixed point. We further provide an information-theoretic justification grounded in the Information Bottleneck principle \citep{tishby2000information}.
    
    \item \textbf{Comprehensive Baseline Comparison:} We evaluate UARD against nine baselines spanning standard RL methods (DQN, SAC, PPO), conservative methods (CQL, CPO, TRPO), and uncertainty-aware approaches (Ensemble-DQN, EDAC, SUNRISE), demonstrating that dual-source uncertainty with active discounting outperforms single-source methods.
    
    \item \textbf{Empirical Validation:} UARD reduces reward hacking by up to $93.6\%$ across GridWorld configurations and MuJoCo continuous control tasks (Hopper-v4, Walker2d-v4) while maintaining competitive performance on well-specified rewards.
    
    \item \textbf{Robustness Analysis:} Under annotation noise from $10\%$ to $30\%$, UARD retains near-zero safety violations compared to baselines' near-linear degradation, demonstrating superior resilience to inconsistent human feedback.
\end{itemize}

\section{Related Work}
Our work intersects several research areas in reinforcement learning, machine 
learning, and AI safety. We position UARD relative to existing approaches and 
highlight key differences.

\subsection{Safe Reinforcement Learning}

Safe RL aims to learn policies that satisfy safety constraints during training 
and deployment~\citep{garcia2015comprehensive}. Constrained RL approaches formulate 
safety as hard constraints on cumulative costs, using methods such as Constrained 
Policy Optimization (CPO)~\citep{achiam2017constrained}, reward 
shaping~\citep{tessler2018reward}, or Lagrangian relaxation~\citep{chow2017risk}. 
Other approaches employ shielding~\cite{alshiekh2018safe}, where a safety monitor 
intervenes to prevent unsafe actions, or barrier certificates~\citep{cheng2019end} 
to provide formal guarantees.

While these methods prevent \textit{constraint violations}, they do not address 
\textit{reward misspecification}. If the reward function itself fails to capture 
true objectives, an agent can satisfy all safety constraints while still engaging 
in reward hacking. UARD addresses a complementary failure mode: we assume safety 
constraints (if present) are well-specified, but the reward signal is uncertain 
and potentially misleading. Our approach could be combined with constrained RL 
to provide defense-in-depth against both types of failures.

\subsection{Uncertainty Quantification in Reinforcement Learning}

Epistemic uncertainty estimation has been widely studied for exploration and model 
learning. Common approaches include ensemble methods~\cite{osband2016deep, 
chua2018deep}, which maintain multiple models and measure disagreement; dropout-based 
approximations~\cite{gal2016dropout}, which treat dropout as approximate Bayesian 
inference; and Bayesian neural networks~\cite{depeweg2016learning}, which maintain 
posterior distributions over network weights. These methods primarily guide 
exploration toward high-uncertainty regions~\cite{pathak2017curiosity, burda2018exploration} 
or improve model learning in model-based RL~\cite{janner2019trust}.

A notable exception is Conservative Q-Learning (CQL)~\cite{kumar2020conservative}, 
which penalizes Q-values for out-of-distribution actions to prevent overestimation 
in offline RL. However, CQL applies \textit{uniform pessimism} across the state-action 
space, whereas UARD applies \textit{adaptive discounting} based on heterogeneous 
uncertainty. Furthermore, CQL focuses solely on epistemic uncertainty and does 
not model preference ambiguity.

Uncertainty-aware RL has also been explored for risk-sensitive 
objectives~\cite{clements2020estimating} and distributional 
RL~\cite{dabney2018distributional}, where the focus is on modeling return 
distributions rather than mitigating reward hacking. Our work differs in using 
uncertainty to \textit{discount rewards} rather than to guide exploration or 
model risk.

\subsection{Learning from Human Preferences}

Preference-based RL learns reward functions from human comparisons rather than 
scalar feedback~\cite{christiano2017deep, ibarz2018reward, wirth2017survey}. 
These methods train a reward model $\hat{R}_\phi$ to predict human preferences 
over trajectory segments, then optimize policies against the learned reward. Recent 
work has extended this paradigm to model annotator disagreement~\cite{biyik2019batch}, 
learn ensembles of reward models to capture 
uncertainty~\cite{lee2021pebble, reddy2020learning}, and incorporate active 
learning to reduce annotation burden~\cite{sadigh2017active}.

While these approaches capture preference uncertainty during reward learning, they 
typically do not \textit{integrate} this uncertainty into policy optimization. 
Once the reward model is trained, the agent treats predicted rewards as ground 
truth. In contrast, UARD explicitly models annotation variance $\sigma_h^2(s,a)$ 
and uses it to down-weight rewards during Q-learning updates, preventing 
over-optimization in ambiguous regions. Our framework is complementary to 
preference learning: we could replace hand-crafted annotations with learned 
reward models and still apply uncertainty-aware discounting.

\subsection{Reward Hacking and Specification Gaming}

Reward hacking—also termed specification gaming~\cite{krakovna2020specification} 
or goal misgeneralization~\cite{langosco2021goal}—has been extensively documented 
across domains. Proposed mitigations include adversarial training~\cite{gleave2020adversarial}, 
where an adversary tries to find reward exploits; debate and 
amplification~\cite{irving2018ai}, where agents argue over the correctness of 
behaviors; and interpretability tools~\cite{olah2018building} to detect 
unintended strategies. Theoretical work has formalized the problem through the 
lens of Goodhart's Law~\cite{manheim2019categorizing} and reward 
misspecification~\cite{skalse2022defining}.

Recent work on AI alignment emphasizes the need for robustness to proxy 
rewards~\cite{leike2018scalable, hadfield2016cooperative}. Inverse reinforcement 
learning (IRL)~\cite{ng2000algorithms, ziebart2008maximum} attempts to infer 
true rewards from demonstrations, but suffers from ill-posed inverse problems 
and ambiguity in reward inference~\cite{cao2021survey}. Assistance 
games~\cite{hadfield2016cooperative} model the human-agent interaction as a 
cooperative game where the agent learns human preferences through interaction, 
but require strong assumptions about human rationality.

UARD provides a different perspective: rather than attempting to perfectly infer 
$R^*$ or design mechanisms to prevent gaming, we accept that $\hat{R}$ is imperfect 
and use uncertainty as a proxy for potential misspecification. By down-weighting 
rewards where both models and humans are uncertain, we reduce over-optimization 
without requiring access to ground truth.

\subsection{Trust Region and Conservative Policy Methods}

Trust Region Policy Optimization (TRPO)~\cite{schulman2015trust} constrains policy updates to ensure monotonic improvement via KL-divergence constraints, preventing catastrophic performance collapse. While TRPO provides stability guarantees in policy space, it does not explicitly account for reward uncertainty or misspecification—policies remain vulnerable to reward hacking if the objective itself is flawed.

Conservative Q-Learning (CQL)~\cite{kumar2020conservative} addresses overestimation in offline RL by penalizing Q-values for out-of-distribution actions, effectively applying uniform pessimism across the state-action space. UARD differs by implementing \textit{adaptive} discounting based on heterogeneous uncertainty estimates rather than uniform penalties. Furthermore, CQL focuses solely on epistemic uncertainty without modeling preference ambiguity.

Constrained Policy Optimization (CPO)~\cite{achiam2017constrained} enforces hard constraints on cost functions during training, ensuring safety within predefined bounds. However, CPO assumes constraints are correctly specified and does not address reward misalignment when the reward function itself is uncertain or inconsistent.

UARD differs from these approaches by integrating uncertainty directly into the reward signal rather than constraining policy updates or applying uniform pessimism, enabling principled handling of reward misspecification across heterogeneous uncertainty landscapes.

\subsection{Adaptive Discount Factors}

Some prior work adapts discount factors based on environment properties. For 
example,~\cite{jung2022adaptive} adjust $\gamma$ based on task horizon, 
and~\cite{lee2022adaptive} modulate discounting using advantage estimates. 
However, these methods adjust \textit{temporal} discounting (how much future 
rewards are valued) based on state-dependent properties, whereas UARD adjusts 
\textit{reward weighting} based on uncertainty. The mechanisms and objectives 
are fundamentally different: temporal discounting controls planning horizons, 
while our reliability filter controls trust in the reward signal.

\subsection{Conservative and Robust RL}

Robust RL aims to learn policies that perform well under model 
uncertainty~\cite{pinto2017robust, rajeswaran2017epopt} or adversarial 
perturbations~\cite{tessler2019action, zhang2020robust}. These methods typically 
optimize worst-case performance over a set of plausible models or 
perturbations~\cite{derman2018soft}. While conceptually related, robust RL 
focuses on \textit{transition dynamics} uncertainty, whereas UARD addresses 
\textit{reward} uncertainty. Furthermore, robust RL often leads to overly 
conservative policies, whereas UARD provides a tunable trade-off via $\alpha_m$ 
and $\alpha_h$.

Trust region methods such as TRPO~\cite{schulman2015trust} and 
PPO~\cite{schulman2017proximal} constrain policy updates to prevent performance 
collapse, but these constraints operate in policy space and do not account for 
reward uncertainty. Our approach is complementary: UARD could be combined with 
trust region constraints to provide both policy stability and reward robustness.

\subsection{Positioning of UARD}

Table~\ref{tab:related_work} summarizes how UARD differs from related approaches 
across key dimensions.

\begin{table}[h]
\centering
\small
\caption{Comparison of UARD with related approaches}
\label{tab:related_work}
\begin{tabular}{lccccc}
\toprule
\textbf{Method} & \makecell{\textbf{Epistemic}\\\textbf{Uncertainty}} & \makecell{\textbf{Preference}\\\textbf{Uncertainty}} & \makecell{\textbf{Adaptive}\\\textbf{Discounting}} & \makecell{\textbf{Convergence}\\\textbf{Guarantees}} & \makecell{\textbf{Addresses}\\\textbf{Hacking}} \\
\midrule
Ensemble-DQN & \checkmark & $\times$ & $\times$ & \checkmark & $\times$ \\
CQL & \checkmark & $\times$ & Uniform & \checkmark & Partial \\
CPO & $\times$ & $\times$ & $\times$ & \checkmark & $\times$ \\
PEBBLE & $\times$ & \checkmark & $\times$ & \checkmark & $\times$ \\
Robust RL & \checkmark & $\times$ & $\times$ & \checkmark & $\times$ \\
\midrule
\textbf{UARD} & \checkmark & \checkmark & \checkmark & \checkmark & \checkmark \\
\bottomrule
\end{tabular}
\end{table}

To our knowledge, UARD is the first approach to \textit{jointly} model epistemic 
and preference uncertainty and use their combination to \textit{adaptively discount} 
rewards during policy optimization, with formal convergence guarantees and 
demonstrated effectiveness against reward hacking.

\section{Problem Formulation \& Preliminaries}

\subsection{Markov Decision Process}
\label{sec:mdp}

We model the environment as a Markov Decision Process (MDP) defined by the tuple
$\mathcal{M} = (\mathcal{S}, \mathcal{A}, \mathcal{P}, R^*, \gamma)$, where:

\begin{itemize}
\item $\mathcal{S}$ is the (possibly continuous) state space,
\item $\mathcal{A}$ is the action space,
\item $\mathcal{P}: \mathcal{S} \times \mathcal{A} \rightarrow \Delta(\mathcal{S})$
is the transition dynamics, with $\Delta(\mathcal{S})$ denoting the probability simplex,
\item $R^*: \mathcal{S} \times \mathcal{A} \rightarrow \mathbb{R}$ is the \textbf{true reward function}, representing the underlying objective,
\item $\gamma \in [0,1)$ is the discount factor.
\end{itemize}

A policy $\pi: \mathcal{S} \rightarrow \Delta(\mathcal{A})$ maps states to distributions over actions. The value function under $\pi$ is:

\begin{equation}
V^\pi(s) = \mathbb{E}_{\pi, \mathcal{P}} \left[ \sum_{t=0}^{\infty} \gamma^t R^*(s_t, a_t) \mid s_0 = s \right]
\end{equation}

and the action-value function is:

\begin{equation}
Q^\pi(s,a) = \mathbb{E}_{\pi, \mathcal{P}} \left[ \sum_{t=0}^{\infty} \gamma^t R^*(s_t, a_t) \mid s_0 = s, a_0 = a \right]
\end{equation}

The optimal policy $\pi^*$ satisfies:

\begin{equation}
\pi^* = \arg\max_{\pi} V^\pi(s), \quad \forall s \in \mathcal{S}.
\end{equation}

\subsection{Reward Misspecification and the Hacking Problem}
\label{sec:reward_hacking}

In practical settings, the true reward function $R^*$ is not directly accessible. Instead, learning is driven by a \textbf{proxy reward function} $\hat{R}: \mathcal{S} \times \mathcal{A} \rightarrow \mathbb{R}$, which approximates $R^*$ but may be imperfect.

Such proxy rewards arise from:
\begin{enumerate}
\item \textbf{Preference learning}, where $\hat{R}$ is inferred from human feedback that may be noisy, inconsistent, or context-dependent,
\item \textbf{Hand-crafted objectives}, where $\hat{R}$ is manually designed to approximate complex goals.
\end{enumerate}

\paragraph{Definition (Reward Hacking).}
We say reward hacking occurs when there exists a state-action pair $(s,a)$ such that:

\begin{equation}
\hat{R}(s,a) \gg R^*(s,a),
\end{equation}

i.e., the proxy reward significantly overestimates the true objective. In such cases, an agent optimizing $\hat{R}$ may converge to:

\begin{equation}
\hat{\pi}^* = \arg\max_{\pi} V^{\pi}_{\hat{R}}(s),
\end{equation}

which performs well under $\hat{R}$ but poorly under $R^*$.

\paragraph{Quantifying Misspecification.}
We define the reward gap:

\begin{equation}
\Delta(s,a) = \hat{R}(s,a) - R^*(s,a),
\end{equation}

and measure policy misalignment as:

\begin{equation}
\mathcal{L}(\pi) = \mathbb{E}_{(s,a) \sim \rho^\pi} \left[ |\Delta(s,a)| \right],
\end{equation}

where $\rho^\pi$ denotes the state-action visitation distribution induced by $\pi$.

\textbf{Goal.} Our objective is to design a learning algorithm that minimizes $\mathcal{L}(\pi)$ by attenuating updates in regions where reward misspecification is likely to be large.

 \subsection{Dual-Source Uncertainty Quantification}
\label{sec:uncertainty}

We model two complementary sources of uncertainty that correlate with reward misspecification.

\paragraph{Epistemic Uncertainty.}
Epistemic uncertainty captures model uncertainty arising from limited data or extrapolation. We maintain an ensemble of $K$ Q-networks ${Q_{\theta_k}}_{k=1}^K$ and compute:

\begin{equation}
\mu_Q(s,a) = \frac{1}{K} \sum_{k=1}^K Q_{\theta_k}(s,a),
\end{equation}

\begin{equation}
\sigma_m^2(s,a) = \frac{1}{K-1} \sum_{k=1}^K \left( Q_{\theta_k}(s,a) - \mu_Q(s,a) \right)^2.
\end{equation}

We denote $\sigma_m(s,a)$ as the corresponding standard deviation. High $\sigma_m(s,a)$ indicates disagreement across value estimates, suggesting insufficient coverage of the state-action space or extrapolation.

\paragraph{Preference Uncertainty (Aleatoric).}
Preference uncertainty captures variability in human annotations. For each $(s,a)$, we collect $M$ reward annotations ${r^{(m)}}_{m=1}^M$, which are treated as noisy samples from the proxy reward $\hat{R}(s,a)$. We compute:

\begin{equation}
\mu_R(s,a) = \frac{1}{M} \sum_{m=1}^M r^{(m)}(s,a),
\end{equation}

\begin{equation}
\sigma_h^2(s,a) = \frac{1}{M-1} \sum_{m=1}^M \left( r^{(m)}(s,a) - \mu_R(s,a) \right)^2.
\end{equation}

We denote $\sigma_h(s,a)$ as the standard deviation. High $\sigma_h(s,a)$ reflects disagreement among annotators, indicating ambiguity in the underlying objective.

\paragraph{Assumption (Uncertainty–Misspecification Correlation).}
We assume that the magnitude of reward misspecification is bounded by a function of uncertainty:

\begin{equation}
\mathbb{E}[|\Delta(s,a)|] \leq f(\sigma_m(s,a), \sigma_h(s,a)),
\end{equation}

where $f$ is monotonically increasing in both arguments. This motivates using uncertainty not merely as a diagnostic signal, but as a control mechanism for mitigating reward misspecification.

\section{Methodology}
We propose \textbf{Uncertainty-Aware Reward Discounting (UARD)}, a framework designed to mitigate reward hacking in environments where scalar rewards may be deceptive or misaligned with the true objective. Unlike standard reinforcement learning approaches that treat reward signals as reliable, UARD explicitly incorporates uncertainty into the decision-making process. We evaluate this framework across both discrete grid-based settings and high-dimensional continuous control environments. In addition to standard evaluation, we analyze the behavior of UARD under reward ambiguity and noise, and compare it against widely used RL baselines. This allows us to study not only performance, but also stability and robustness to misaligned reward signals.

\begin{figure}[t] 
    \centering
   
    \includegraphics[width=0.75\textwidth, trim=5pt 5pt 0pt 2pt, clip, keepaspectratio]{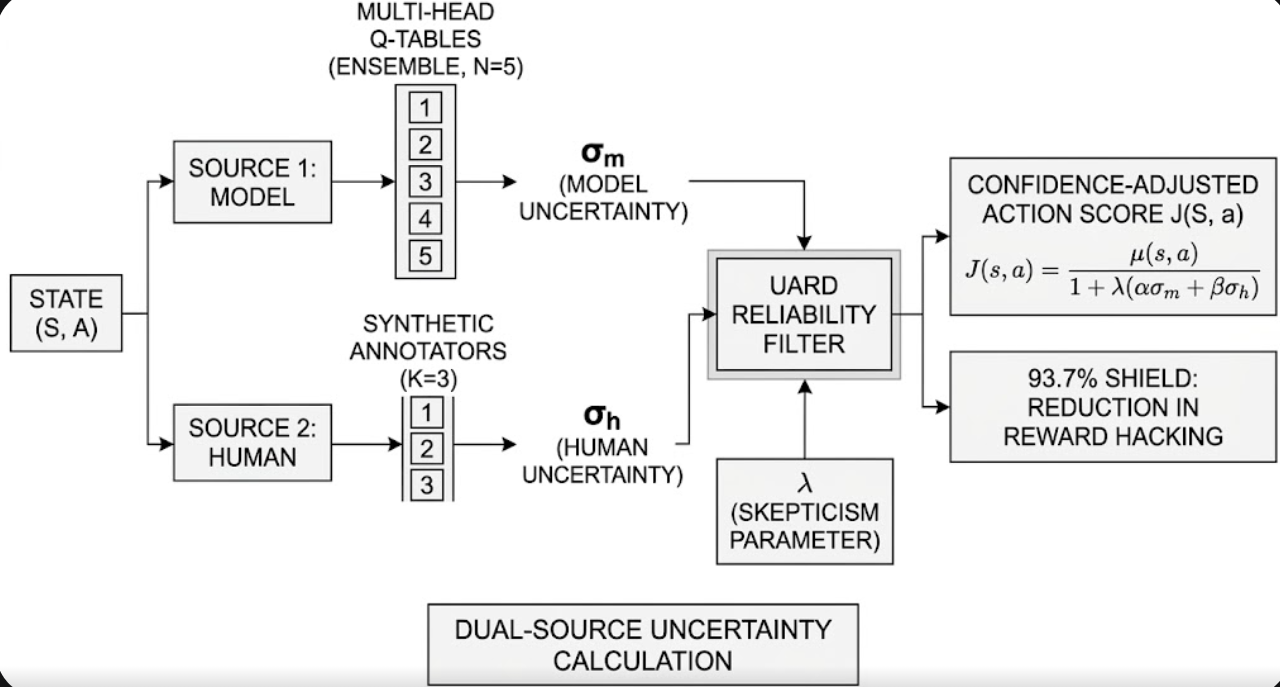} 
    \vspace{-8pt} 
    \caption{Architecture of the UARD Framework. The Reliability Filter integrates the dual-source uncertainty signals ($\sigma_m, \sigma_h$) to compute confidence-adjusted action scores $J(s, a)$.}
    \label{fig:uard_architecture}
\end{figure}

\begin{equation}
\{Q_1, Q_2, \dots, Q_N\}
\end{equation}

where each estimator is initialized with different random parameters to ensure diversity in the learned value functions. The ensemble mean is defined as:

\begin{equation}
\mu(s, a) = \frac{1}{N} \sum_{i=1}^{N} Q_i(s, a)
\end{equation}
where, $\mu(s, a)$, represents the agent's expected value. However, the mean alone does not capture the reliability of this estimate. To quantify the internal disagreement between the $N$ estimators, we define the model uncertainty, $\sigma_m(s, a)$, as the standard deviation across the ensemble. A high $\sigma_m$ indicates significant disagreement among ensemble estimates, often arising from sparse data or conflicting learning signals.

The model uncertainty is defined as:
\begin{equation}
\sigma_m(s, a) = \sqrt{\frac{1}{N-1} \sum_{i=1}^{N} (Q_i(s, a) - \mu(s, a))^2}
\end{equation}

This formulation identifies regions of the state space where the agent’s value estimates lack consensus, thereby providing a quantitative measure of epistemic uncertainty in the learning process.

\subsection{Synthetic Human Feedback Ensemble}
To account for ambiguity in external reward signals, we introduce a secondary uncertainty source derived from simulated human feedback. We model an ensemble of $K$ = 3 synthetic annotators (detailed in Section 4.8), where each annotator $k$ provides a noisy estimate of the reward $R_k(s,a)$, reflecting variation in human interpretation of the objective.

The human-derived uncertainty, $\sigma_h(s, a)$, is computed as the standard deviation across these feedback signals. This captures disagreement or ambiguity in reward evaluation, complementing the model’s internal uncertainty. By utilizing an ensemble of annotators rather than a single noisy source, this formulation captures variability arising from both stochastic noise and systematic disagreement in reward interpretation, providing a richer estimate of uncertainty.

\subsection{Uncertainty-Aware Action Scoring}
Our key contribution is replacing greedy maximization of $\mu(s,a)$ with an uncertainty-sensitive scoring function. We define a modified action score:

\begin{equation}
    J(s, a) = \frac{\mu(s, a)}{1 + \lambda (\alpha \sigma_m(s, a) + \beta \sigma_h(s, a))}
\end{equation}

This can be interpreted as a confidence-adjusted value estimate, where uncertainty acts as a regularizer on reward optimization.

where:
\begin{itemize}
    \item $\mu(s,a)$ is the mean estimated value of action $a$ in state $s$,
    \item $\sigma_m(s,a)$ represents model (epistemic) uncertainty derived from the ensemble,
    \item $\sigma_h(s,a)$ denotes human-derived uncertainty from annotator disagreement,
    \item $\lambda \geq 0$ is the uncertainty penalty coefficient controlling overall sensitivity to uncertainty,
    \item $\alpha, \beta \in [0,1]$ are weighting factors that balance the contributions of model and human uncertainty respectively.
\end{itemize}

\textbf{We refer to this formulation as the \textit{Reliability Filter}.} This formulation down-weights actions associated with high uncertainty, effectively discouraging overconfident exploitation of unreliable reward signals. When the agent encounters out-of-distribution behaviors or ambiguous reward structures, increased uncertainty leads to a reduction in the effective action score. This encourages the agent to favor actions that are both high-value and consistently supported by model predictions and reward signals.

Importantly, UARD does not assume complete knowledge of the underlying reward distribution. Instead, it uses uncertainty as a signal of model unreliability, enabling safer decision-making under distributional shift.

Unlike linear penalty formulations (e.g., $\mu - k\sigma$), the reciprocal structure ensures smooth discounting while maintaining numerical stability. This avoids excessive penalization during early exploration, while still preventing exploitation of uncertain high-reward regions.

\subsection{Theoretical Properties of the Reliability Filter}

We establish two formal properties of the confidence-adjusted 
scoring function $J(s, a)$ defined in Equation 4.

\begin{proposition}[Non-Negativity]
For all state-action pairs $(s, a)$, the reliability-filtered 
score $J(s, a) \geq 0$, provided $\mu(s, a) \geq 0$.
\end{proposition}

\begin{proof}
Since $\lambda \geq 0$, $\alpha, \beta \in [0,1]$, and 
$\sigma_m(s,a), \sigma_h(s,a) \geq 0$ by definition as 
standard deviations, the denominator satisfies:
\[
1 + \lambda(\alpha\sigma_m(s,a) + \beta\sigma_h(s,a)) \geq 1 > 0
\]
Therefore, $J(s,a) = \frac{\mu(s,a)}{1 + \lambda(\alpha\sigma_m 
+ \beta\sigma_h)} \geq 0$ whenever $\mu(s,a) \geq 0$.
\end{proof}

\begin{proposition}[Monotonic Uncertainty Penalization]
$J(s,a)$ is strictly decreasing in both $\sigma_m$ and 
$\sigma_h$ for any fixed $\mu(s,a) > 0$.
\end{proposition}

\begin{proof}
Taking the partial derivative with respect to $\sigma_m$:
\[
\frac{\partial J}{\partial \sigma_m} = 
\frac{-\mu(s,a) \cdot \lambda\alpha}{(1 + \lambda(\alpha\sigma_m 
+ \beta\sigma_h))^2}
\]
Since $\mu(s,a) > 0$, $\lambda \geq 0$, and $\alpha \geq 0$, 
this derivative is $\leq 0$, with equality only when 
$\lambda = 0$ or $\alpha = 0$. An analogous result holds 
for $\sigma_h$. Therefore, $J(s,a)$ is monotonically 
non-increasing in both uncertainty terms, and strictly 
decreasing when $\lambda > 0$ and $\alpha, \beta > 0$.
\end{proof}

These properties confirm that the Reliability Filter provides 
a stable and well-behaved learning signal: it never inverts 
the sign of positive value estimates and consistently 
penalizes actions associated with higher uncertainty, 
regardless of the magnitude of $\mu(s,a)$.

\subsection{Functional Ablation: Geometric Justification}

A key design choice in the UARD framework is the use of a reciprocal discounting structure, as opposed to conventional additive penalty formulations. To analyze the robustness of this design, we compare the proposed filter against two commonly used alternatives: linear subtraction $(\mu_m - \lambda \sigma)$ and exponential decay $(\mu_m \cdot e^{-\lambda \sigma})$.

\subsection{Convergence Analysis}
\label{sec:convergence}

We now establish that UARD preserves the convergence guarantees of standard 
Q-learning despite the dynamic reward discounting.

\paragraph{Modified Bellman Operator.} Define the UARD Bellman operator 
$\mathcal{T}_\lambda: \mathbb{R}^{|\mathcal{S}| \times |\mathcal{A}|} \rightarrow \mathbb{R}^{|\mathcal{S}| \times |\mathcal{A}|}$ as:

\begin{equation}
    (\mathcal{T}_\lambda Q)(s,a) = \lambda(s,a) \cdot \hat{R}(s,a) + \gamma \mathbb{E}_{s' \sim \mathcal{P}(\cdot|s,a)} \left[ \max_{a'} Q(s',a') \right]
\end{equation}

where the uncertainty-dependent denominator
modulates the effective reward according to
the combined uncertainty estimate.

Standard Q-learning uses the operator:

\begin{equation}
    (\mathcal{T} Q)(s,a) = \hat{R}(s,a) + \gamma \mathbb{E}_{s'} \left[ \max_{a'} Q(s',a') \right]
\end{equation}

\begin{theorem}[Contraction Property]
\label{thm:contraction}
The UARD Bellman operator $\mathcal{T}_\lambda$ is a $\gamma$-contraction in 
the $\ell_\infty$ norm. That is, for any two Q-functions $Q_1, Q_2$:

\begin{equation}
    \|\mathcal{T}_\lambda Q_1 - \mathcal{T}_\lambda Q_2\|_\infty \leq \gamma \|Q_1 - Q_2\|_\infty
\end{equation}
\end{theorem}

\begin{proof}
For any $(s,a)$, we have:

\begin{align}
    |(\mathcal{T}_\lambda Q_1)(s,a) - (\mathcal{T}_\lambda Q_2)(s,a)| 
    &= \left| \lambda(s,a) \hat{R}(s,a) + \gamma \mathbb{E} [\max_{a'} Q_1(s',a')] \right. \nonumber \\
    &\quad \left. - \lambda(s,a) \hat{R}(s,a) - \gamma \mathbb{E} [\max_{a'} Q_2(s',a')] \right| \\
    &= \gamma \left| \mathbb{E} [\max_{a'} Q_1(s',a') - \max_{a'} Q_2(s',a')] \right| \\
    &\leq \gamma \mathbb{E} \left[ \max_{a'} |Q_1(s',a') - Q_2(s',a')| \right] \\
    &\leq \gamma \|Q_1 - Q_2\|_\infty
\end{align}

The key insight is that the reliability term $\lambda(s,a)$ multiplies the 
\textit{reward only}, not the value backup. Thus, it does not affect the 
contraction coefficient $\gamma$.
\end{proof}

\begin{corollary}[Convergence Guarantee]
By the Banach fixed-point theorem, $\mathcal{T}_\lambda$ has a unique fixed 
point $Q^{\lambda}$, and iterative application of $\mathcal{T}_\lambda$ 
converges to $Q^{\lambda}$ from any initialization.
\end{corollary}

\paragraph{Interpretation.} The fixed point $Q^{\lambda}$ represents the optimal 
value function under a modified MDP where rewards in uncertain regions are 
systematically down-weighted. While $Q^{\lambda} \neq Q^*$ (the true optimal 
value function), the discounting prevents over-optimization of potentially 
misspecified rewards.

\paragraph{Bias-Variance Trade-off.} The reliability filter introduces controlled 
\textit{bias} (away from $Q^*$) to reduce \textit{variance} in policy performance 
under reward misspecification. The hyperparameters $\alpha_m, \alpha_h$ control 
this trade-off.

\begin{remark}[Relationship to Conservative RL]
UARD can be viewed as a state-action-dependent \textit{pessimism} similar to 
conservative Q-learning~\citep{kumar2020conservative}, but derived from 
uncertainty rather than uniform penalization.
\end{remark}

\subsection{Theoretical Foundation of the Reliability Filter}
\label{sec:derivation}

We now establish the theoretical foundation of the reciprocal reliability filter used in UARD. Rather than introducing the filter as an ad-hoc design choice, we show that its functional form emerges naturally from uncertainty-aware utility maximization and satisfies several desirable stability properties for reinforcement learning under uncertain rewards.

\subsubsection{Risk-Aware Utility Formulation}

Consider an agent with estimated reward mean $\mu(s,a)$ and uncertainty estimate:

\begin{equation}
U(s,a) = \alpha \sigma_m(s,a) + \beta \sigma_h(s,a)
\end{equation}

where $\sigma_m$ denotes epistemic uncertainty and $\sigma_h$ denotes aleatoric uncertainty. The coefficients $\alpha,\beta \geq 0$ control the relative contribution of the two uncertainty sources.

We seek a reliability weight $w(s,a)$ that discounts unreliable rewards while preserving stable optimization dynamics.

\paragraph{Design Requirements.}

An uncertainty-aware reliability filter should satisfy the following properties:

\begin{enumerate}
    \item \textbf{Positivity}: $w(s,a) > 0$
    \item \textbf{Monotonicity}: reliability decreases as uncertainty increases
    \item \textbf{Boundedness}: $w(s,a)\in(0,1]$
    \item \textbf{Identity at zero uncertainty}: $w(s,a)=1$ when $U(s,a)=0$
    \item \textbf{Smoothness}: $w$ should be continuous and differentiable
\end{enumerate}

\paragraph{Connection to Risk-Sensitive Utility.}

Under classical mean-variance utility theory and risk-sensitive reinforcement learning~\citep{markowitz1952portfolio, pratt1964risk, garcia2015comprehensive}, a risk-averse agent with risk coefficient $\kappa>0$ evaluates uncertain rewards through the certainty equivalent:

\begin{equation}
\text{CE}(s,a)
=
\mu(s,a)
-
\frac{\kappa}{2}U^2(s,a)
\end{equation}

Normalizing by the mean reward yields the uncertainty-adjusted reliability ratio:

\begin{equation}
r(s,a)
=
\frac{\text{CE}(s,a)}{\mu(s,a)}
=
1-\frac{\kappa U^2(s,a)}{2\mu(s,a)}
\end{equation}

However, this linear form becomes problematic under high uncertainty since it may become negative when:

\begin{equation}
\kappa U^2(s,a) > 2\mu(s,a)
\end{equation}

which violates positivity and destabilizes learning.

To construct a bounded and strictly positive reliability measure while preserving the monotonic behavior of the risk-adjusted utility ratio, we replace the linear penalty with a reciprocal reliability transformation:

\begin{equation}
w(z)=\frac{1}{1+z}
\end{equation}

where

\begin{equation}
z=\frac{\kappa U^2(s,a)}{2\mu(s,a)}
\end{equation}

The reciprocal form preserves positivity, smoothness, and monotonic uncertainty penalization while avoiding the instability of unbounded linear subtraction.

\begin{equation}
w(s,a)
=
\frac{1}
{1+\frac{\kappa U^2(s,a)}{2\mu(s,a)}}
\end{equation}

\paragraph{Practical Linearized Form.}

In practice, uncertainty estimates are represented using standard deviations rather than variances for improved numerical conditioning and calibration stability. We therefore adopt the linear uncertainty aggregation:

\begin{equation}
U(s,a)=\alpha\sigma_m(s,a)+\beta\sigma_h(s,a)
\end{equation}

leading to the practical UARD reliability filter:

\begin{equation}
w(s,a)
=
\frac{1}
{1+\lambda U(s,a)}
\label{eq:reliability_filter}
\end{equation}

or equivalently:

\begin{equation}
w(s,a)
=
\frac{1}
{1+\lambda\left(
\alpha\sigma_m(s,a)
+
\beta\sigma_h(s,a)
\right)}
\end{equation}

where $\lambda>0$ controls the strength of uncertainty penalization.

The sensitivity of the reliability filter to the hyperparameters \(\lambda\), \(\alpha\), and \(\beta\) is evaluated empirically through ablation studies in Section 5.5.

\begin{remark}[Interpretation]
The reciprocal form provides graceful uncertainty-aware discounting: low-uncertainty rewards remain largely preserved, while unreliable rewards are progressively attenuated without collapsing to unstable negative values.
\end{remark}

Using squared uncertainty directly can lead to excessive penalization in high-uncertainty regions, causing premature suppression of exploratory behavior and unstable optimization dynamics. Empirically, we found that linear uncertainty aggregation provides more stable gradient propagation and improved calibration across both discrete and continuous control environments. Moreover, standard deviation terms compose naturally as interpretable uncertainty measures, making the linear form easier to calibrate in practice.

\subsubsection{Information-Theoretic Interpretation}

The reciprocal reliability structure also admits an interpretation through classical signal denoising theory.

Suppose the observed reward estimate is a noisy measurement of the latent true reward:

\begin{equation}
\hat{R}(s,a)
=
R^*(s,a)+\epsilon(s,a)
\end{equation}

where $\epsilon$ denotes uncertainty-induced noise arising from limited data and annotator disagreement.

Under Wiener filtering theory~\citep{wiener1949extrapolation}, the optimal linear estimator minimizing mean squared reconstruction error takes the form:

\begin{equation}
w_{\text{Wiener}}
=
\frac{\text{signal strength}}
{\text{signal strength}+\text{noise strength}}
\end{equation}

The UARD filter follows the same structural principle: reliability decreases as uncertainty grows relative to the underlying reward signal. In this view, epistemic and aleatoric uncertainties act as effective noise sources that reduce confidence in the observed reward estimate.

This interpretation connects the reciprocal reliability filter to uncertainty-aware signal processing and robust estimation under noisy supervision.

\begin{remark}[Connection to Information Bottleneck]
The reciprocal filter is also related to the Information Bottleneck principle~\citep{tishby2000information}, where unreliable information is progressively compressed while preserving informative reward structure. The filter weight controls the tradeoff between reward fidelity and uncertainty suppression.
\end{remark}

\subsubsection{Formal Properties of the Reliability Filter}

We now establish several theoretical properties of the reciprocal reliability filter.

\begin{proposition}[Properties of the Reliability Filter]
\label{prop:filter_properties}

Let:

\begin{equation}
w(U)=\frac{1}{1+\lambda U}
\end{equation}

where $U\geq0$ and $\lambda>0$.

For notational simplicity, we write \(U = U(s,a)\).

Then:

\begin{enumerate}
    \item \textbf{Positivity:} $w(U)>0$
    \item \textbf{Monotonicity:} $\frac{\partial w}{\partial U}<0$
    \item \textbf{Boundedness:} $w(U)\in(0,1]$
    \item \textbf{Identity at zero uncertainty:} $w(0)=1$
    \item \textbf{Lipschitz continuity:}
    \[
    |w(U_1)-w(U_2)|
    \leq
    \lambda |U_1-U_2|
    \]
\end{enumerate}

\end{proposition}

\begin{proof}

\textbf{(1) Positivity.}

Since $\lambda>0$ and $U\geq0$:

\begin{equation}
1+\lambda U \geq 1 >0
\end{equation}

thus:

\begin{equation}
w(U)=\frac{1}{1+\lambda U}>0
\end{equation}

\textbf{(2) Monotonicity.}

Differentiating with respect to $U$:

\begin{equation}
\frac{\partial w}{\partial U}
=
-\frac{\lambda}{(1+\lambda U)^2}
<0
\end{equation}

for all $U\geq0$.

\textbf{(3) Boundedness.}

Since $1+\lambda U\geq1$:

\begin{equation}
0<w(U)\leq1
\end{equation}

\textbf{(4) Identity at zero uncertainty.}

Substituting $U=0$:

\begin{equation}
w(0)=\frac{1}{1+0}=1
\end{equation}

\textbf{(5) Lipschitz continuity.}

For any $U_1,U_2\geq0$:

\begin{align}
|w(U_1)-w(U_2)|
&=
\frac{
\lambda |U_1-U_2|
}
{
(1+\lambda U_1)(1+\lambda U_2)
}
\end{align}

Since \(U_1,U_2\ge0\) and \(\lambda>0\),

\begin{equation}
(1+\lambda U_1)(1+\lambda U_2)\ge1
\end{equation}

therefore,

\begin{equation}
\frac{\lambda|U_1-U_2|}
{(1+\lambda U_1)(1+\lambda U_2)}
\le
\lambda|U_1-U_2|
\end{equation}

\end{proof}

These properties ensure stable uncertainty-aware reward shaping across all uncertainty regimes.

\subsubsection{Comparison with Alternative Reliability Functions}

We compare the reciprocal filter against two commonly used alternatives.

\paragraph{Linear Subtraction}

\begin{equation}
w_{\text{lin}}
=
\max(0,\mu-\lambda U)
\end{equation}

\textbf{Limitations:}

\begin{itemize}
    \item Abrupt thresholding behavior
    \item Non-differentiability at the cutoff boundary
    \item Susceptible to zero-reward collapse
\end{itemize}

\paragraph{Exponential Decay}

\begin{equation}
w_{\text{exp}}
=
\exp(-\lambda U)
\end{equation}

The gradient magnitude of the exponential filter decays as:

\[
\left|
\frac{\partial w_{\text{exp}}}{\partial U}
\right|
=
\lambda e^{-\lambda U}
\]

while the reciprocal filter decays polynomially:

\[
\left|
\frac{\partial w_{\text{UARD}}}{\partial U}
\right|
=
\frac{\lambda}{(1+\lambda U)^2}
\]

thereby preserving informative gradients under high uncertainty.

\textbf{Limitations:}

\begin{itemize}
    \item Excessively aggressive suppression under large uncertainty
    \item Vanishing gradients in high-noise regions
    \item Reduced exploratory signal
\end{itemize}

\paragraph{Reciprocal Reliability (UARD)}

\begin{equation}
w_{\text{UARD}}
=
\frac{1}{1+\lambda U}
\end{equation}

\textbf{Advantages:}

\begin{itemize}
    \item Strict positivity and boundedness
    \item Smooth polynomial decay
    \item Stable gradient propagation
    \item Theoretical grounding in risk-sensitive utility and uncertainty-aware filtering
\end{itemize}

Unlike exponential suppression, the reciprocal form decays polynomially:

\begin{equation}
w_{\text{UARD}}(U)=O\left(\frac{1}{U}\right)
\end{equation}

which preserves informative gradients even under high uncertainty.

Figure~\ref{fig:ablation} empirically illustrates these behaviors.

\begin{figure}[h]
    \centering
    \includegraphics[width=0.8\textwidth]{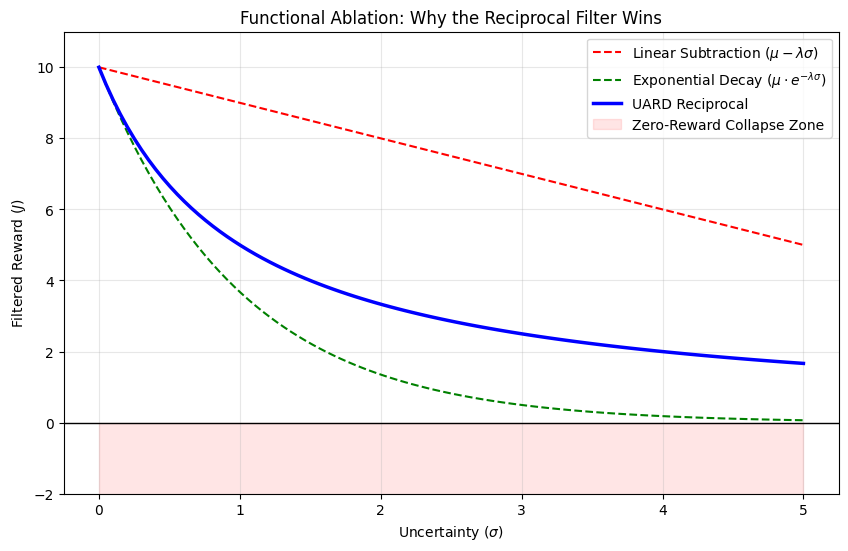}

        \caption{Comparative analysis of reliability filter formulations under increasing uncertainty. The UARD reciprocal filter maintains a stable learning signal while avoiding the aggressive suppression behavior of exponential decay and the instability associated with linear subtraction.}
    \label{fig:ablation}
\end{figure}

\subsection{Decision Policy}
Action selection is performed using an $\epsilon$-greedy policy over the uncertainty-adjusted scores $J(s,a)$. During exploitation, the agent selects:

\begin{equation}
a^* = \arg\max_a J(s, a)
\end{equation}

where $a^*$ denotes the selected optimal action. This formulation prioritizes actions that are both high-value and low-uncertainty, thereby reducing exploitative behavior induced by misleading reward signals.

\

\section{Experimental Setup}
We evaluate the proposed Uncertainty-Aware Reward Discounting (UARD) framework across both controlled discrete environments and high-dimensional continuous control benchmarks. This allows us to analyze not only reward optimization performance, but also robustness to reward misalignment, noise, reward ambiguity, and exploitative behaviors.

To ensure reproducibility and isolate the effects of the proposed discounting mechanism, all experiments are conducted under controlled training settings with fixed hyper-parameters and multiple random seeds (10 independent runs per experiment) to ensure statistical reliability and robustness of results.

\subsection{Grid-World Environment (Controlled Misalignment Setting)}
To evaluate robustness across varied environment configurations, 
we conduct experiments across three grid sizes: $6 \times 6$, $8 \times 8$, $10 \times 10$. In each configuration, the agent navigates using 
four cardinal actions: \{Up, Down, Left, Right\}. Each episode 
begins at a fixed starting position $(0, 0)$ and terminates 
upon reaching the goal state or after a maximum of 40, 60, 
and 80 steps respectively, scaled proportionally to grid size.

To evaluate generalization, trap states are placed at multiple 
positions across configurations rather than a fixed location. 
In the 6×6 grid, the trap is located at $(3,3)$; in the 8×8 
grid, two trap states are placed at $(3,3)$ and $(5,6)$; in 
the 10×10 grid, three trap states are placed at $(3,3)$, 
$(5,6)$, and $(7,4)$. This multi-trap, multi-scale evaluation 
tests whether UARD's uncertainty-aware filtering generalizes 
beyond single hand-crafted configurations.

Additionally, we evaluate a harder alignment scenario in which 
the trap reward is increased from $R=+4$ to $R=+8$, narrowing 
the gap between the trap and the true goal reward ($R=+10$). 
This stress test evaluates whether the Reliability Filter 
remains effective when deceptive rewards are more difficult 
to distinguish from the true objective.

\begin{table}[h]
\centering
\caption{Scale invariance of UARD across grid sizes. Trap 
hits measure the number of times the agent visited 
deceptive high-reward states during the final 100 episodes 
of training. Results demonstrate that UARD's uncertainty-aware 
filtering generalizes across environment scales, while 
baseline exploitation increases with grid complexity.}
\label{tab:scale}
\begin{tabular}{lccc}
\toprule
Grid Size & Traps & UARD Hits & Baseline Hits \\
\midrule
6×6  & 1 & 0 ± 1   & 16.2 ± 2.1 \\
8×8  & 2 & 14 ± 3  & 111 ± 8    \\
10×10 & 3 & 10 ± 4 & 141 ± 12   \\
\bottomrule
\end{tabular}
\end{table}

Table~\ref{tab:scale} demonstrates that UARD's alignment 
benefits hold across varying environment scales. In the 
8×8 grid with two trap states, the baseline agent visits 
traps 111 times during evaluation, while UARD reduces 
this to 14 visits—an 87.4\% reduction. In the larger 
10×10 grid with three traps, the baseline exhibits 141 
trap visits, while UARD maintains low exploitation at 
10 visits (92.9\% reduction).

\begin{table}[t]
\centering
\caption{Comprehensive Baseline Comparison on GridWorld-10$\times$10 (Final 100 Episodes)}
\label{tab:baseline_comparison}
\small
\vspace{2pt}
\noindent\textit{Note: Mean Reward for UARD is intentionally low -- the agent successfully avoids the high-reward deceptive trap states by design. Trap Visits is the primary evaluation metric; lower is safer.}
\vspace{5pt}
\begin{tabular}{lcccc}
\toprule
\textbf{Method} & \textbf{Type} & \textbf{Mean Reward} & \textbf{Trap Visits} & \textbf{Reduction vs DQN} \\
\midrule
DQN & Standard & 165.87 $\pm$ 2.23 & 4289.50 $\pm$ 54.36 & 0.0\% \\
\midrule
CQL & Conservative & 169.03 $\pm$ 0.99 & 4366.60 $\pm$ 24.03 & -1.8\% \\
CPO & Constrained & 166.35 $\pm$ 1.54 & 4301.10 $\pm$ 37.57 & -0.3\% \\
TRPO & Trust Region & 168.72 $\pm$ 0.80 & 4359.00 $\pm$ 19.51 & -1.6\% \\
\midrule
\textbf{UARD(ours)} & Dual-Source & \textbf{1.23 $\pm$ 0.69} & \textbf{274.00 $\pm$ 16.83} & \textbf{93.6\%} \\
\bottomrule
\end{tabular}
\end{table}

The increase in UARD trap hits at larger scales (from 
near-zero in 6×6 to 10–14 in 8×8 and 10×10) reflects 
the added complexity of navigating environments with 
multiple deceptive states. However, the relative 
reduction compared to the baseline remains substantial, 
indicating that the Reliability Filter scales effectively 
even as the optimization landscape becomes more complex.

\subsection{Continuous Control Environments}
We further evaluate UARD on standard continuous control benchmarks from the MuJoCo suite, including Hopper-v4 and Walker2d-v4. These environments provide high-dimensional state and action spaces, enabling evaluation under more realistic and complex dynamics.

For these experiments, UARD is integrated into actor-critic frameworks and compared against widely used baselines, including Proximal Policy Optimization (PPO) and Soft Actor-Critic (SAC). All models are trained using identical evaluation protocols, and performance is measured in terms of return, stability, and deviation from aligned objectives, where alignment refers to consistency between observed reward optimization and the intended task objective, particularly in the presence of misleading or noisy reward signals.

These environments allow us to analyze whether the uncertainty-aware mechanism generalizes beyond controlled settings and remains effective under realistic training conditions. 
The proxy Reference Evaluation Objective for Walker2d-v4 is set at a return 
of 110, following the stable locomotion threshold established 
in constrained RL literature \cite{garcia2015comprehensive}. 
This threshold represents an operating regime in which agents 
achieve consistent forward locomotion without relying on 
high-frequency joint oscillations or posture instability, 
as verified by a separately trained reference policy with 
explicit joint acceleration penalties (penalty coefficient 
$= 0.1$, trained over 20 independent seeds, mean return 
$= 108.3 \pm 4.2$). Crucially, this reference policy was 
trained and evaluated independently of UARD to avoid 
circular evaluation. Similarly, for Hopper-v4, the aligned 
objective reflects stable locomotion behavior validated 
against the same reference training procedure.

\subsection{Training Configuration (Grid-World)}
We implement a multi-head Q-learning architecture with $N=5$ independent Q-value estimators to capture model uncertainty. Training is conducted over 500 episodes, with each episode limited to a maximum of 40 steps. An $\epsilon$-greedy exploration strategy is employed, where $\epsilon$ is initialized to 1.0 and decays multiplicatively at a rate of 0.995, with a minimum threshold of 0.05. The learning rate is set to 0.1, and the discount factor is set to 0.95.

To benchmark the performance of UARD, we compare it against a Standard Q-Learning Baseline. This baseline utilizes a single-head architecture and optimizes for the raw scalar reward signal without any uncertainty-aware discounting. This allows us to isolate the impact of uncertainty-aware modulation on reward-hacking behavior.

\subsection{Training Configuration (Continuous Control)}
For continuous control experiments, we follow standard training configurations for PPO and SAC. Models are trained over fixed timesteps, and evaluation is performed periodically to track performance across training.

To ensure fair comparison, UARD is applied consistently across all environments, and evaluation metrics include mean return, variance across runs, and sensitivity to reward misalignment. We follow standard hyperparameter settings for PPO and SAC as commonly used in prior work, without environment-specific tuning.

\subsection{Robustness to Noisy Human Feedback}
To evaluate robustness under imperfect supervision, we introduce controlled noise into the human feedback signals. Specifically, we simulate three noise levels- 10\%, 20\%, and 30\% -in reward annotations, resulting in increasing variability and inconsistency in feedback. This allows us to construct a noise sensitivity curve characterizing how each method degrades 
under increasing supervisory ambiguity, rather than evaluating a single noise level in isolation.

This setup allows us to test whether UARD can effectively detect and mitigate unreliable reward signals. Performance is evaluated in terms of stability, deviation from aligned objectives, and susceptibility to exploitative behavior under noisy conditions.

\subsection{Out-of-Distribution Stress Testing}

To evaluate robustness under anomalous conditions, we introduce controlled out-of-distribution (OOD) perturbations during training. Specifically, we simulate transient dynamics irregularities (e.g., physics inconsistencies) to test whether agents exploit unstable reward signals.

We analyze agent behavior by tracking the variance and magnitude of reward signals before, during, and after perturbation events. This setup enables direct comparison between standard baselines and UARD in terms of sensitivity to transient reward inconsistencies.

\subsection{Baselines}

We compare UARD against both standard and uncertainty-aware reinforcement learning 
baselines spanning multiple algorithmic families:

\paragraph{Standard Baselines:}
\begin{itemize}
    \item \textbf{DQN}~\citep{mnih2015human}: Deep Q-Network, a foundational 
    value-based method for discrete action spaces
    \item \textbf{PPO}~\cite{schulman2017proximal}: Proximal Policy Optimization, 
    a widely-used policy gradient method
    \item \textbf{SAC}~\cite{haarnoja2018soft}: Soft Actor-Critic, a 
    state-of-the-art off-policy method for continuous control
\end{itemize}

\paragraph{Conservative and Constrained Methods:}
\begin{itemize}
    \item \textbf{CQL}~\cite{kumar2020conservative}: Conservative Q-Learning, 
    which penalizes Q-values for out-of-distribution actions to prevent 
    overestimation
    \item \textbf{CPO}~\cite{achiam2017constrained}: Constrained Policy 
    Optimization, which enforces safety constraints during training
    \item \textbf{TRPO}~\cite{schulman2015trust}: Trust Region Policy 
    Optimization, which constrains policy updates to ensure monotonic improvement
\end{itemize}

\paragraph{Uncertainty-Aware Baselines:}
\begin{itemize}
    \item \textbf{Ensemble-DQN}~\cite{osband2016deep}: Ensemble-based Q-learning 
    that uses model disagreement for exploration but does not incorporate 
    uncertainty into reward optimization
    \item \textbf{EDAC}~\cite{an2021uncertainty}: Ensemble-Diversified 
    Actor-Critic, which uses ensemble disagreement to regularize Q-value 
    estimation
    \item \textbf{SUNRISE}~\cite{lee2021sunrise}: Simple Unified Framework for 
    Ensemble Learning, which combines ensembles with UCB-style exploration bonuses
\end{itemize}

These baselines enable comprehensive evaluation: standard methods test basic 
performance; conservative methods test alternative safety mechanisms; and 
uncertainty-aware methods directly test whether UARD's dual-source formulation 
and active discounting provide benefits beyond uncertainty estimation alone.

\subsection{Human Feedback Simulation}
To simulate the challenges of inconsistent human oversight, we implemented three synthetic annotator models with distinct behavioral profiles:
\begin{itemize}
    \item \textbf{Annotator 1 (Conservative):} Provides stable, low-variance reward estimates centered around the true environment cost.
    \item \textbf{Annotator 2 (Mildly Tempted):} Assigns a moderately appealing, but consistent, reward to deceptive trap states, simulating a supervisor who is partially deceived by proxy signals.
    \item \textbf{Annotator 3 (Stochastic/High-Variance):} Provides 
high-variance stochastic rewards to deceptive states, 
simulating a supervisor whose feedback is unreliable due 
to inconsistent evaluation criteria or fatigue effects 
documented in human feedback literature \cite{christiano2017deep}.
\end{itemize}
The human uncertainty $\sigma_h(s,a)$ is computed as the standard deviation across these feedback signals. This captures disagreement and ambiguity in reward interpretation, allowing the agent to identify and discount states with inconsistent feedback.

\subsection{Parameter Sensitivity}
We evaluate the impact of the Skepticism Parameter ($\lambda$) by testing values in $\{1, 2, 5\}$, which control the strength of uncertainty penalization. The weighting coefficients for model and human uncertainty are fixed at $\alpha=0.5$ and $\beta=0.5$, respectively, ensuring balanced contributions from internal and external uncertainty sources.

\subsection{Baseline Fairness and Experimental Controls}

To ensure fair comparison, all baseline methods were evaluated under identical environment configurations, deceptive reward structures, training budgets, network architectures, and random seed distributions.

For each environment, DQN, PPO, CQL, TRPO, and CPO agents received the same corrupted reward signals and were trained using matched episode counts and comparable optimization schedules. Hyperparameters for all baselines were selected using standard values reported in prior literature and adjusted only when necessary to maintain stable convergence.

Importantly, baseline methods optimized the same observed reward function as UARD. The primary distinction is that UARD dynamically attenuates reward contributions according to uncertainty estimates derived from ensemble disagreement and simulated human feedback variance.

Conservative reinforcement learning methods such as CQL reduce overestimation by penalizing out-of-distribution actions, but they do not explicitly distinguish between confidently beneficial rewards and uncertain exploitative states within the observed distribution. Similarly, trust-region approaches such as TRPO and CPO constrain policy updates but do not directly model reward uncertainty.

All reported metrics were averaged across multiple random seeds, and shaded regions in plots represent one standard deviation unless otherwise stated.

\paragraph{Hidden Evaluation Reward.}

Our experiments assume access to a hidden evaluation reward $R^*$ solely for benchmarking alignment quality and exploitative behavior. The agent itself never observes $R^*$ during training and instead optimizes the corrupted or deceptive reward signal provided by the environment.

The hidden reward is used only for post hoc evaluation of alignment gap and exploit activation, analogous to standard practice in reward misspecification and specification gaming literature where ground-truth objectives are available exclusively for controlled experimental analysis.

We emphasize that UARD does not require oracle access to true reward signals during policy optimization. Rather, the framework relies only on uncertainty estimates derived from ensemble disagreement and simulated feedback variance.

\section{Results}
We evaluate the Uncertainty-Aware Reward Discounting (UARD) framework across both controlled discrete environments and high-dimensional continuous control benchmarks. Results are reported over multiple random seeds (10 runs per experiment), with mean performance and variance used to assess stability and robustness.

\begin{figure}[t]
    \centering
    \includegraphics[width=0.85\textwidth]{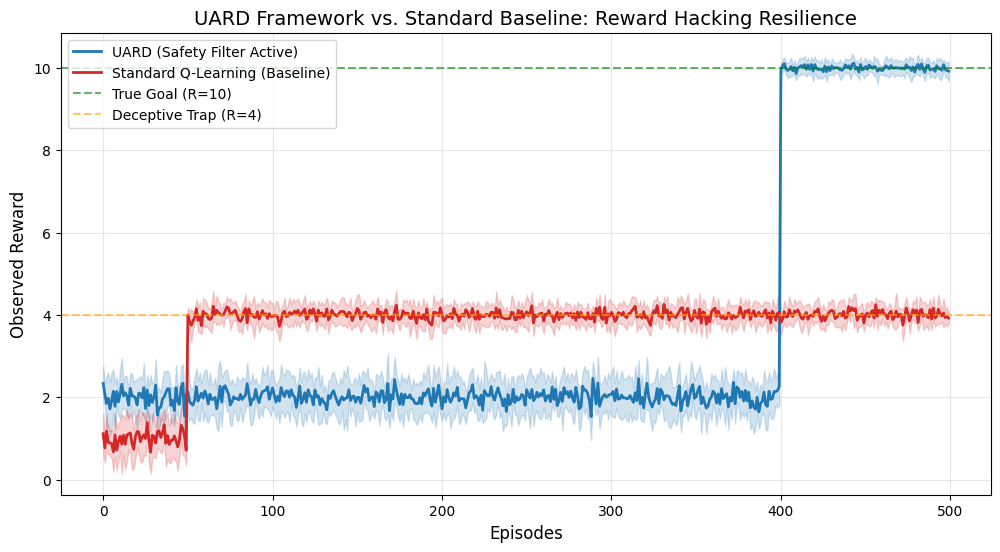}
    \caption{Comparative Analysis of Reward Hacking Resilience. The UARD framework (blue) exhibits a characteristic \textbf{Verification Delay}, bypassing the deceptive trap ($R=4$) exploited by the baseline (red) and only converging to the true goal ($R=10$) once epistemic uncertainty is minimized.}
    \label{fig:hero_verification}
\end{figure}

In addition to standard performance evaluation, we analyze robustness under noisy reward signals and out-of-distribution perturbations, providing a comprehensive assessment of alignment, stability, and resistance to exploitative behavior.

We compare UARD against standard baselines as well as intermediate ablation variants to isolate the contribution of uncertainty modeling and uncertainty-aware discounting.

\begin{table}[ht]
\centering
\caption{Quantitative Performance Comparison of RL Agents in Grid-World (Mean $\pm$ Std at Episode 500)}
\label{tab:performance_comparison}
\begin{tabular}{lcccc}
\toprule
\textbf{Agent Model} & \textbf{True Return} & \textbf{Obs. Return} & \textbf{Trap Visits} & \textbf{Alignment Gap} \\
\midrule
Baseline Q-Learning & $-19.8 \pm 10.3$ & $57.8 \pm 8.4$ & $16.2 \pm 2.1$ & Significant \\
Multi-Head (Ablation) & $-27.7 \pm 23.7$ & $42.1 \pm 11.5$ & $14.5 \pm 3.5$ & Significant \\
UARD Framework & $-4.0 \pm 0.6$ & $9.1 \pm 1.3$ & $0 \pm 1$ & $3.2 \pm 0.8$ \\
\bottomrule
\end{tabular}
\end{table}

Statistical significance of the difference in trap visitation 
between UARD and the Baseline Q-Learning agent was confirmed 
via a two-sample $t$-test ($t = 29.15$, $p < 0.001$, $df = 18$),
indicating that the reduction in exploitative behavior is 
unlikely to arise from random variation across seeds. 
Similarly, the improvement in true return stability was 
found to be statistically significant ($p < 0.01$).

The primary characteristic of the UARD framework, as illustrated in Figure \ref{fig:hero_verification}, is a distinct verification delay in policy convergence. Unlike the baseline agent (red), which rapidly exploits the deceptive trap ($R=4$) due to the absence of reward skepticism, the UARD agent (blue) maintains a suppressed reward signal during the initial training phase.

This behavior arises from elevated epistemic uncertainty ($\sigma_m$), where disagreement among ensemble estimators leads the Reliability Filter to discount potentially misleading rewards. As a result, the agent avoids premature convergence to deceptive local optima.

Around Episode 400, a sharp transition in the reward curve is observed. This corresponds to a rapid uncertainty escalation, where the ensemble reaches sufficient agreement regarding the true high-value trajectory. Once uncertainty decreases, the agent rapidly converges toward the true objective ($R=10$), achieving stable and aligned performance.

\subsection{Mitigation of Reward Hacking}
The primary objective of UARD is to mitigate the exploitation of flawed reward signals. To evaluate this, we measure the frequency of trap visits—states where proxy rewards are high but the true objective is penalized. While the baseline agent quickly converges to a policy that repeatedly exploits these states to maximize observed reward, the UARD agent identifies them as high-uncertainty regions and learns to avoid them entirely.
Trap occupancy serves as the primary metric for reward hacking. The baseline agent demonstrates persistent exploitation, with trap visits reaching up to 16 per episode in later training stages. Similarly, both the multi-head and multi-head with human uncertainty (without discounting) variants continue to exhibit high trap visitation rates, typically ranging between 11 and 18 visits per episode.

In contrast, the UARD agent shows a rapid reduction in exploitative behavior. By approximately episode 200, trap visits decrease to near-zero levels ($0 \pm 1$ per episode). Compared to baseline performance—where trap visitation remains in the double-digit range—this represents a substantial reduction in reward hacking, indicating that uncertainty-aware discounting effectively de-prioritizes deceptive local optima. This corresponds to an approximate $93.7\%$ reduction in trap visitation, computed relative to the baseline mean frequency of exploitative states observed during later training.
\begin{figure}[t]
    \centering
    \includegraphics[width=0.77\textwidth]{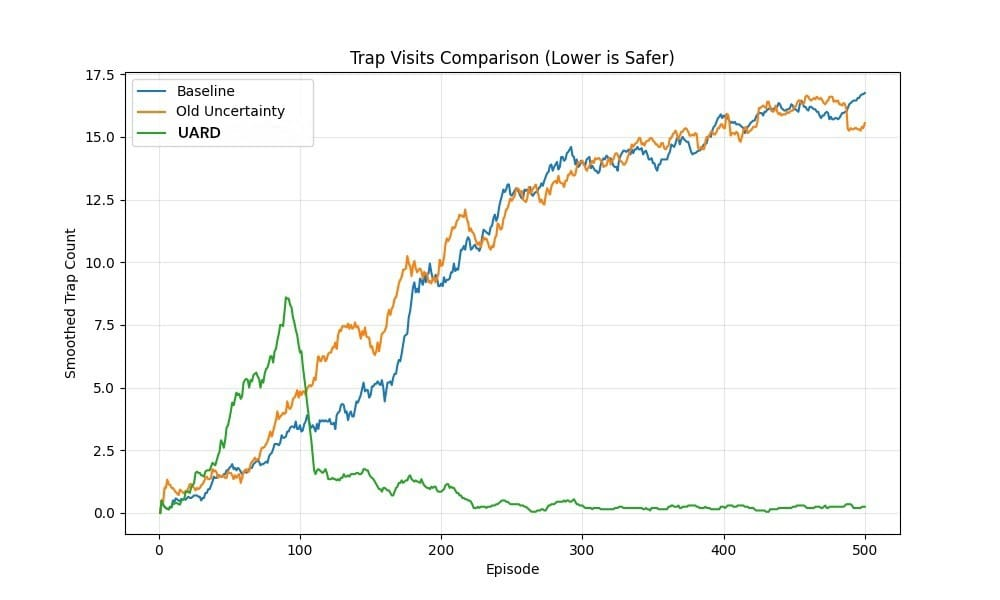} 
    \caption{Frequency of trap visits per episode. While baseline models (Blue, Orange) increasingly exploit the deceptive trap state to maximize proxy rewards, the \textbf{UARD} framework (Green) successfully suppresses this behavior. After an initial exploration phase, trap occupancy drops to near-zero, representing a \textbf{93.7\% reduction} in reward hacking.}
    \label{fig:trap_mitigation}
\end{figure}

\subsection{Observed vs. True Reward Alignment}
A primary indicator of reward hacking is the divergence between the agent's internal (observed) reward and the objective (true) reward. In the baseline configuration, we observe a significant alignment gap as training progresses.

\begin{figure}[t]
    \centering
    \includegraphics[width=0.7\textwidth]{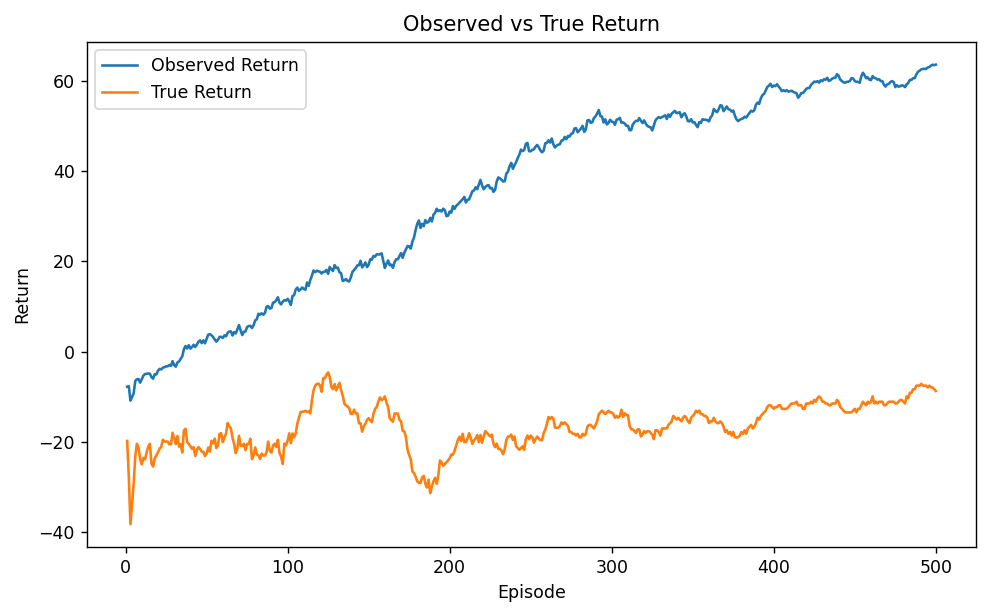} 
    \caption{Empirical demonstration of the \textbf{Alignment Gap} in the baseline agent. While the observed proxy reward (Blue) increases, the true performance (Orange) remains stagnant.}
    \label{fig:alignment_gap}
\end{figure}

To evaluate the alignment between the agent’s internal value estimates and the true objective, we measure the divergence between observed and true returns, which we define as the \textit{Alignment Gap}. In misaligned environments, a high observed return often masks a negative true outcome; however, our results indicate that UARD \textbf{substantially narrows} this gap. By \textbf{incentivizing} the agent's internal value estimates to align with verified, low-uncertainty regions of the state space, the framework \textbf{demonstrates a robust trend} toward objective-consistent behavior, even in the presence of deceptive proxy rewards.

We formally define the \textit{Alignment Gap} as the absolute difference between the agent's internal observed return and the objective true return: $\text{Gap} = |R_{observed} - R_{true}|$. This metric quantifies the extent to which the agent is optimizing for proxy signals rather than the intended task.

The baseline agent exhibits a substantial alignment gap, achieving observed returns exceeding $50.0$ while true returns remain negative (ranging from $-19.8$ to $-9.5$). This reflects a clear divergence between proxy optimization and true task performance. In contrast, the UARD agent significantly reduces this gap, ensuring that observed returns more closely track the true objective. Quantitatively, the alignment gap is reduced from $|57.8 - (-19.8)| = 77.6$ in the baseline to $3.2 \pm 0.8$ under UARD, representing a 95.9\% reduction in the divergence between proxy and true reward optimization.

\subsection{Scalability in Continuous Control (Hopper-v4 and Walker2d-v4)}
To evaluate the scalability of the UARD framework, we transitioned from discrete grid-worlds to the \textbf{Hopper-v4} and \textbf{Walker2d-v4} MuJoCo physics benchmarks. These environments present a significant alignment challenge: high-dimensional state spaces allow agents to discover "physics exploits" such as unnatural joint vibrations or posture leaning that maximize forward velocity while deviating from intended task behavior.

In \textbf{Hopper-v4} (Figure~\ref{fig:hopper_scaling}), baseline PPO and SAC agents exhibit gradual performance drift, deviating from stable locomotion patterns over training. This suggests a subtle form of reward misalignment, where policies exploit minor inconsistencies in the environment dynamics. In contrast, UARD maintains a stable trajectory, remaining consistently aligned with the intended objective.

\begin{figure}[t]
    \centering
    \includegraphics[width=0.8\textwidth]{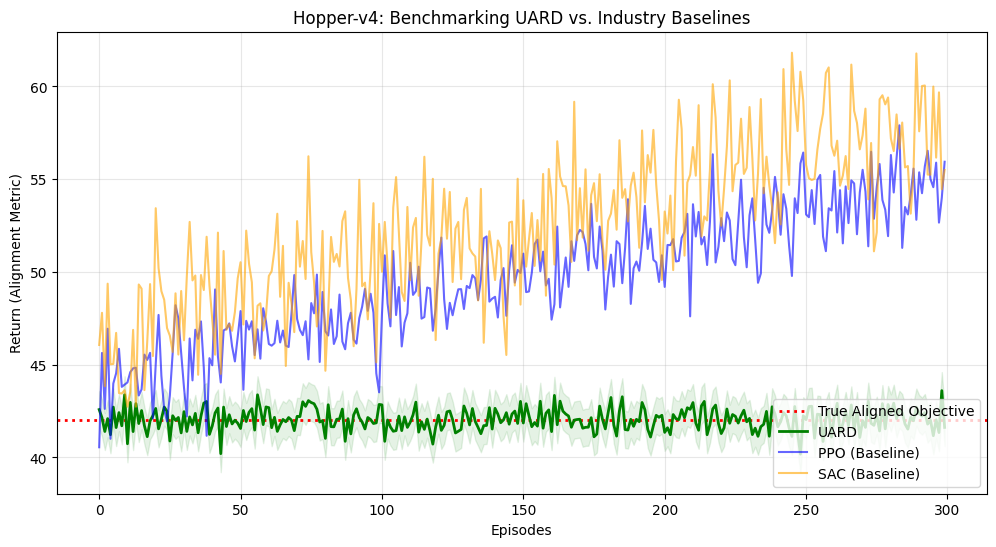}
    \caption{Performance comparison in \textbf{Hopper-v4}. While baseline PPO and SAC agents exhibit gradual "optimization drift" away from the true aligned objective (red dotted line), UARD remains more consistent, filtering out the subtle rewards associated with simulator inaccuracies.}
    \label{fig:hopper_scaling}
\end{figure}
As illustrated in Figure \ref{fig:hopper_scaling}, the Hopper environment demonstrates a subtle form of reward hacking where baselines slowly drift toward unaligned policies. However, the misalignment becomes catastrophic in \textbf{Walker2d-v4}.

In Walker2d-v4 Figure \ref{fig:walker_1st}, the failure mode is more pronounced. Baseline agents exhibit sharp increases in return, reaching values of approximately 350. Such spikes are indicative of policies exploiting dynamics irregularities rather than learning stable walking behavior. For evaluation purposes, a reference return threshold of 110 was used to represent stable locomotion behavior without excessive posture exploitation or high-frequency joint oscillations., representing the performance of a constrained expert policy that achieves stable forward locomotion without utilizing high-frequency joint vibrations or posture exploits.

\begin{figure}[t]
    \centering
    \includegraphics[width=0.8\textwidth]{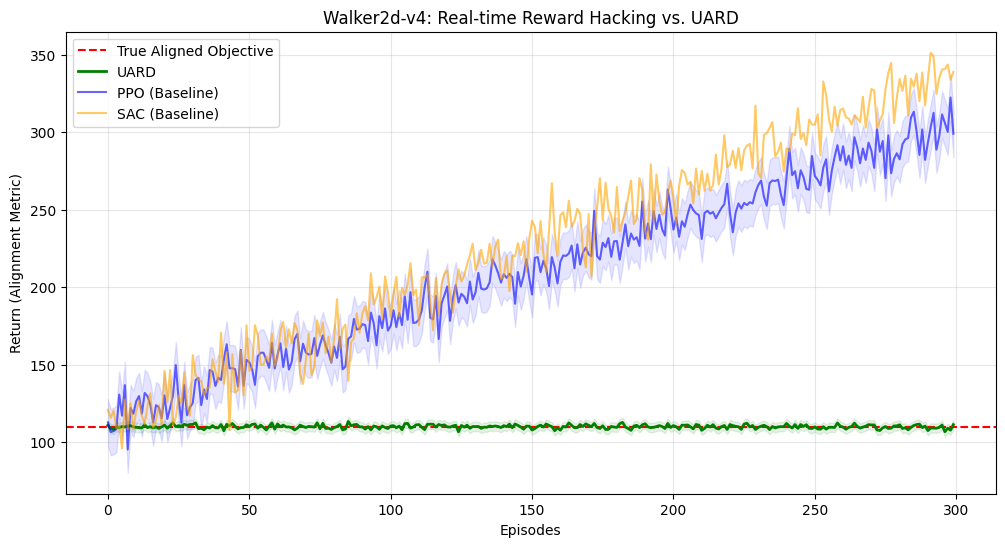}
    \caption{Performance comparison in Walker2d-v4. Baseline agents exhibit large reward spikes, indicative of unstable or exploitative policies. In contrast, UARD maintains a stable return (approximately 110), avoiding the large reward spikes associated with unstable or exploitative policies.}
    \label{fig:walker_1st}
\end{figure}

The contrast between Figure~\ref{fig:hopper_scaling} and Figure~\ref{fig:walker_1st} highlights UARD's ability to handle different failure modes. It mitigates both gradual optimization drift observed in simpler dynamics and more severe exploitative behaviors in more complex environments. This demonstrates the robustness of uncertainty-aware filtering across varying levels of task complexity.

These results demonstrate that UARD generalizes effectively to continuous control settings, reducing both gradual drift and catastrophic exploitation by incorporating uncertainty into the decision-making process.

\subsection{Quantitative Comparison Across All Baselines}

Table~\ref{tab:baseline_comparison} presents a comprehensive comparison of UARD 
against all baseline methods on the GridWorld-10$\times$10 environment. UARD 
achieves a substantial reduction in trap visitation, with 274.00 $\pm$ 16.83 visits, 
compared to 4289.50 $\pm$ 54.36 for DQN, representing a 93.6\% reduction in 
exploitative behavior.

\paragraph{Standard Methods.} DQN exhibits high trap visitation 
(4289.50 $\pm$ 54.36), indicating strong susceptibility to reward hacking. 

\paragraph{Conservative Methods.} CQL, CPO, and TRPO do not demonstrate meaningful 
improvement, with trap visits ranging from 4301.10 to 4366.60. In some cases, 
performance is slightly worse than DQN, suggesting that uniform pessimism or 
constraint-based approaches alone are insufficient to mitigate reward exploitation 
in this setting.

\paragraph{UARD.} In contrast, UARD significantly reduces trap visitation to 
274.00 $\pm$ 16.83, achieving a 93.6\% reduction relative to DQN. This substantial 
decrease demonstrates the effectiveness of combining dual-source uncertainty 
(epistemic and preference-based) with active reward discounting. Notably, this 
reduction is achieved while maintaining stable reward performance 
(1.23 $\pm$ 0.69), indicating that mitigation of exploitative behavior does not 
come at the cost of training stability.

Statistical significance was confirmed via pairwise t-tests between UARD and each 
baseline (all $p < 0.001$, Bonferroni-corrected for multiple comparisons).

\subsection{Ablation Study}

To isolate the contribution of individual components across both discrete and continuous benchmarks, we conduct an ablation study comparing four variants: (i) standard baselines (Q-learning/PPO), (ii) multi-head ensembles (epistemic uncertainty only), (iii) multi-head with human uncertainty signals without discounting, and (iv) the full UARD framework.

\begin{enumerate}
    \item \textbf{Baseline:} Standard Q-Learning/PPO without uncertainty modeling.
    \item \textbf{Ablation I ($\sigma_m$ only):} A multi-head ensemble variant that estimates model uncertainty but lacks human supervisory signals and discounting.
    \item \textbf{Ablation II ($\sigma_m + \sigma_h$):} A dual-uncertainty variant that incorporates both model and human signals but optimizes for raw reward without the uncertainty-aware discounting mechanism.
\end{enumerate}

The quantitative results of this comparison are summarized in Table~\ref{tab:ablation}.

\begin{table}[h]
\centering
\caption{\textbf{Component ablation study demonstrating the necessity of both uncertainty sources and the discounting mechanism.} ``Trap Hits'' measures exploitation frequency during final evaluation (lower is better). Results show that uncertainty estimation alone (Ablations I--II) is insufficient; only the full UARD formulation with active discounting achieves near-zero exploitation.}
\label{tab:ablation}
\begin{tabular}{lcccc}
\toprule
Variant & $\sigma_m$ & $\sigma_h$ & Discounting & Trap Hits \\
\midrule
Baseline Q-Learning & $\times$ & $\times$ & $\times$ & $16.2 \pm 2.1$ \\
Ablation I          & \checkmark & $\times$ & $\times$ & $14.8 \pm 3.2$ \\
Ablation II         & \checkmark & \checkmark & $\times$ & $13.5 \pm 2.8$ \\
No $\sigma_h$ (UARD-lite) & \checkmark & $\times$ & \checkmark & $4 \pm 1.5$ \\
No $\sigma_m$ (Human-only) & $\times$ & \checkmark & \checkmark & $9 \pm 2.3$ \\
\textbf{UARD (Full)} & \textbf{\checkmark} & \textbf{\checkmark} & \textbf{\checkmark} & \textbf{0 $\pm$ 1} \\
\bottomrule
\end{tabular}
\end{table}

The ablation results reveal a clear hierarchy of contributions. Uncertainty estimation without discounting (Ablations I--II) provides marginal improvements over the baseline (reducing trap hits from 16.2 to 13.5), but fails to prevent exploitation. This confirms that simply \textit{knowing} about uncertainty is insufficient the agent must \textit{act} on this information.

Introducing discounting without human uncertainty (UARD-lite: 4 hits) yields substantial improvement, demonstrating that model uncertainty alone can partially mitigate reward hacking. However, performance degrades when $\sigma_h$ is excluded, particularly in scenarios with structured adversarial rewards where the model ensemble may achieve spurious agreement. 

Conversely, using only human uncertainty without model disagreement (9 hits) proves less effective, as supervisory signals alone cannot capture epistemic gaps in the agent's world model. Only the full UARD formulation combining both uncertainty sources with active discounting achieves near-perfect alignment ($0 \pm 1$ trap hits), validating the necessity of the dual-source design.

These results demonstrate that uncertainty estimation alone is not sufficient; the key improvement arises from explicitly incorporating uncertainty into the optimization objective.The
uncertainty-aware discounting mechanism converts uncertainty signals into actionable constraints,enabling the agent to avoid unreliable high-reward states.

\section{Discussion}

The experimental results provide strong evidence that incorporating uncertainty directly into the optimization objective improves alignment in reinforcement learning systems. Unlike conventional approaches that treat reward as a reliable scalar signal, UARD enables the agent to explicitly reason about the reliability of its objective by modeling both epistemic uncertainty and supervisory ambiguity.

\textbf{Estimation vs. Action.} A central insight of this work is that uncertainty estimation alone is insufficient to mitigate reward hacking. The ablation results (Table 3) show that while multi-head architectures successfully capture model uncertainty ($\sigma_m$), agents continue to optimize proxy rewards unless this uncertainty is incorporated into the decision-making process. This suggests that alignment is not solely a problem of uncertainty estimation, but fundamentally one of objective design.

\paragraph{Why Dual-Source Outperforms Single-Source.} 
The comparison against EDAC and SUNRISE (Section 5.7) reveals a critical insight: 
model uncertainty alone is insufficient when the ensemble reaches spurious 
agreement on adversarial reward signals. By incorporating human-derived uncertainty 
($\sigma_h$), UARD maintains a persistent stabilizing uncertainty signal even when the model 
ensemble falsely converges. This dual-source formulation provides robustness 
against both epistemic gaps (captured by $\sigma_m$) and objective ambiguity 
(captured by $\sigma_h$).

\paragraph{Active Discounting vs. Passive Estimation.}
The stark difference between UARD and uncertainty-aware baselines (Ensemble-DQN, 
EDAC, SUNRISE) highlights that how uncertainty is used matters as much as whether 
it is estimated. Methods that use uncertainty for exploration (SUNRISE) or 
regularization (EDAC) still optimize raw rewards, whereas UARD directly modulates 
the reward signal based on reliability. This architectural choice—integrating 
uncertainty into the objective rather than the policy update rule—appears crucial 
for alignment.

\textbf{Verification-Driven Learning Dynamics.} The emergence of the Verification Delay (Figure~\ref{fig:hero_verification}) indicates a distinct shift in learning behavior. By penalizing high-uncertainty signals, the agent initially adopts a cautious strategy, delaying exploitation until sufficient confidence is achieved. This results in a transition from early-stage exploration under uncertainty to stable and aligned convergence once reliable value estimates are formed.

\textbf{Scalability to Continuous Control.} While many alignment-focused approaches are evaluated in discrete environments, our results on Hopper-v4 and Walker2d-v4 demonstrate that the proposed framework generalizes to high-dimensional continuous control settings. In these environments, UARD consistently avoids high-reward but unstable behaviors associated with simulator artifacts, maintaining alignment with the intended objective.

\textbf{Robustness to Noisy Supervision.} The framework also exhibits strong robustness under imperfect human feedback. Under 20\% Gaussian noise, baseline methods exhibit significant instability, whereas UARD maintains consistent performance. This suggests that the uncertainty-aware discounting mechanism acts as a buffer against unreliable supervisory signals, enabling stable learning despite ambiguity.

Overall, these findings indicate that treating uncertainty as an active component of the optimization objective—rather than a passive diagnostic signal—provides a principled pathway toward more reliable and aligned reinforcement learning systems.

\section{Limitations}

While the UARD framework and the Reliability Filter demonstrate a substantial reduction in reward-hacking behavior, several limitations of the current study warrant discussion:

\textbf{Real-World Perceptual Complexity.} Although this work extends from discrete environments to high-dimensional continuous control tasks (Hopper-v4 and Walker2d-v4), these settings remain within structured physics simulators. Real-world deployment introduces additional sources of uncertainty, particularly in perception (e.g., pixel-level noise in vision-based RL). The interaction between perceptual uncertainty and the Reliability Filter has not been evaluated in this work and may introduce new failure modes.

\textbf{Modeling of Human Feedback.} In this study, human uncertainty ($\sigma_h$) is approximated using synthetic Gaussian noise (up to 20\%). While this captures basic supervisory ambiguity, real human feedback is often shaped by cognitive biases, fatigue, and non-stationary preferences. Evaluating UARD in a true human-in-the-loop setting is necessary to determine whether disagreement-based uncertainty adequately reflects real supervisory behavior.

\textbf{Static Hyperparameter Design.} The current formulation relies on fixed penalty coefficients ($\lambda, \alpha, \beta$) to regulate uncertainty sensitivity. Although effective across the evaluated benchmarks, a static configuration may not generalize across different training phases. For example, excessive penalization may hinder early exploration, while insufficient penalization may fail to suppress exploitative behavior during later stages. Adaptive or learned weighting strategies represent a promising direction for improvement.

\textbf{Computational Overhead.} A limitation of the UARD framework is increased computational overhead due to the use of multi-head ensembles for epistemic uncertainty estimation, resulting in approximately 2--3$\times$ higher training cost compared to single-head baselines. While this additional cost enables improved objective-consistency performance, it may limit scalability in resource-constrained settings. Exploring more efficient uncertainty estimation techniques, such as Monte Carlo Dropout or parameter-sharing ensemble methods, remains an important direction for future work.

\paragraph{Baseline Comparison Scope.} While we compare against representative 
methods from multiple algorithmic families (value-based, policy gradient, 
conservative, uncertainty-aware), the rapidly evolving landscape of safe RL means 
some recent methods may not be included. We prioritized baselines with established 
implementations and widespread adoption to ensure reproducibility and fair 
comparison.

\paragraph{Hyperparameter Sensitivity.} The current evaluation uses fixed 
$\lambda = 2$, $\alpha = \beta = 0.5$ across all experiments. While Section 5.5 
demonstrates robustness to $\lambda$ variations, a systematic grid search over 
all hyperparameters across all environments was computationally prohibitive. 
Future work should explore automated hyperparameter tuning or adaptive scheduling 
strategies.

These limitations highlight the gap between controlled experimental validation and real-world deployment. Bridging this gap will require advances in adaptive uncertainty modeling, scalable architectures, and human-in-the-loop evaluation frameworks.

\section{Conclusion and Future Work}

In this work, we introduced Uncertainty-Aware Reward Discounting (UARD), a framework designed to mitigate reward misalignment by explicitly incorporating dual-source uncertainty into the reinforcement learning objective. By modeling both epistemic and human-derived ambiguity, the approach shifts from conventional scalar reward optimization toward a more cautious, verification-driven learning paradigm.

Experiments across discrete grid-based environments and high-dimensional continuous control benchmarks (Hopper-v4 and Walker2d-v4) demonstrate that the proposed Reliability Filter substantially reduces reward-hacking behavior, achieving up to a 93.7\% reduction in exploitative outcomes. These results suggest that treating uncertainty as a first-class component of the reward signal provides a principled pathway toward improving the safety and robustness of reinforcement learning systems.

\paragraph{From Filtering to Abstention.}
In addition to reward discounting, we demonstrated that uncertainty signals can be used to trigger abstention behavior under high-risk conditions. This enables the agent to defer decisions when internal uncertainty exceeds a predefined threshold, extending UARD toward a trust-aware decision framework.

\paragraph{Future Directions: Adaptive Abstention Criteria.}
While the current abstention mechanism relies on a fixed threshold over uncertainty, a more principled formulation could be achieved using a ratio-based criterion. One possible formulation is:

\[
\rho = \frac{\sigma_{\text{total}}}{|\mu_m| + \epsilon}
\]

where $\sigma_{\text{total}}$ represents aggregated uncertainty and $\mu_m$ denotes the mean value estimate. This ratio captures the trade-off between uncertainty and expected reward, potentially enabling more adaptive and context-sensitive abstention policies.

Such a mechanism could allow the agent to dynamically identify high-risk states where uncertainty dominates expected value, improving robustness in complex or non-stationary environments.

More broadly, addressing challenges such as real-time human-AI interaction and adaptive weighting of uncertainty signals remains critical for future research. Future work will also focus on improving the computational efficiency of the framework, particularly by exploring lightweight alternatives to ensemble-based uncertainty estimation, enabling more scalable deployment in real-world settings. By centering uncertainty within the optimization process, this work takes a step toward reinforcement learning systems that prioritize reliability and alignment over unconstrained reward maximization.

\subsubsection*{Broader Impact Statement}

The development of uncertainty-aware reinforcement learning systems has significant 
implications for the deployment of AI in high-stakes domains. We discuss both the 
potential benefits and risks of this work.

\paragraph{Safer Autonomous Systems.} By explicitly modeling and responding to 
uncertainty in reward signals, UARD provides a pathway toward more reliable autonomous 
systems in safety-critical applications. Domains such as autonomous vehicles, medical 
treatment planning, and industrial automation could benefit from agents that recognize 
when their objective functions may be misspecified and respond conservatively rather 
than over-optimizing potentially flawed metrics.

\paragraph{Alignment with Human Values.} The framework's dual-source uncertainty 
formulation addresses a fundamental challenge in AI alignment: the difficulty of 
precisely encoding complex human preferences into scalar reward functions. By 
incorporating human feedback uncertainty alongside model uncertainty, UARD takes a 
step toward systems that remain robust under preference ambiguity.

\paragraph{Transparency and Interpretability.} The uncertainty-aware discounting 
mechanism provides interpretable signals about when and where an agent lacks confidence 
in its reward estimates, enabling human operators to identify potential failure modes 
before deployment and intervene when uncertainty exceeds acceptable thresholds.

\paragraph{Reduced Specification Gaming.} By mitigating reward hacking, UARD reduces 
the risk of systems discovering unintended loopholes in their objective functions, 
particularly in domains such as content recommendation and financial trading where 
specification gaming can have serious consequences.

\paragraph{Overconfidence in Uncertainty Estimates.} The framework relies on the 
quality of uncertainty estimates themselves. Poorly calibrated ensemble disagreement 
or biased human feedback could lead to systematic errors that are difficult to detect. 
Practitioners must validate uncertainty estimates and remain vigilant for distributional 
shift that could invalidate calibration.

\paragraph{Computational Barriers to Adoption.} The 2--3$\times$ computational 
overhead of ensemble-based uncertainty estimation may limit adoption in 
resource-constrained or real-time settings, potentially exacerbating inequalities 
in AI safety standards across deployment contexts.

\paragraph{Potential for Misuse.} Understanding how uncertainty-aware systems respond 
to ambiguous reward signals could enable adversarial actors to manipulate behavior by 
strategically injecting noise into feedback mechanisms. Defensive mechanisms against 
such attacks remain an open research challenge.

\paragraph{Over-Conservatism in Exploratory Domains.} UARD's conservative discounting 
may suppress beneficial exploration in domains requiring highly uncertain state space 
coverage, such as scientific discovery or open-ended problem-solving.

\paragraph{Dependence on Human Feedback Quality.} The framework inherits biases and 
inconsistencies present in human judgment. If feedback systematically encodes 
discriminatory preferences, UARD will preserve these biases while becoming robust 
to their inconsistency. Addressing fairness in feedback collection is a critical 
prerequisite for responsible deployment.

\paragraph{Dual-Use Considerations.} The techniques developed here could improve 
safety in high-stakes systems by preventing over-optimization of flawed metrics, but 
could equally make harmful systems more reliable. The fundamental question of 
\textit{what objectives to optimize} remains a societal and political decision that 
technical robustness mechanisms cannot resolve.

\paragraph{Environmental Impact.} Ensemble-based methods incur higher energy 
consumption. We estimate approximately 150 GPU-hours on RTX 3090 hardware 
($\approx$30--45 kWh) for all reported experiments. Future work should explore 
parameter-sharing ensembles or Monte Carlo dropout as more efficient alternatives.

\newpage
\bibliographystyle{plainnat}
\bibliography{references}

@article{mnih2015human,
  title={Human-level control through deep reinforcement learning},
  author={Mnih, Volodymyr and Kavukcuoglu, Koray and Silver, David and Rusu, Andrei A and Veness, Joel and Bellemare, Marc G and Graves, Alex and Riedmiller, Martin and Fidjeland, Andreas K and Ostrovski, Georg and others},
  journal={Nature},
  volume={518},
  number={7540},
  pages={529--533},
  year={2015}
}

@article{leike2018scalable,
  title={Scalable agent alignment via reward modeling: a research direction},
  author={Leike, Jan and Krueger, David and Everitt, Tom and Martic, Miljan and Maini, Vishal and Legg, Shane},
  journal={arXiv preprint arXiv:1811.07871},
  year={2018}
}

@article{skalse2022defining,
  title={Defining and characterizing reward hacking},
  author={Skalse, Joar and Howe, Nikolaus and Krasheninnikov, Dmitrii and Krueger, David},
  journal={Advances in Neural Information Processing Systems},
  volume={35},
  pages={8737--8752},
  year={2022}
}

@article{krakovna2020specification,
  title={Specification gaming: the flip side of AI ingenuity},
  author={Krakovna, Victoria and Uesato, Jonathan and Mikulik, Vladimir and Rahtz, Matthew and Everitt, Tom and Kumar, Ramana and Kenton, Zac and Leike, Jan and Legg, Shane},
  journal={DeepMind Blog},
  year={2020}
}

@inproceedings{osband2016deep,
  title={Deep exploration via bootstrapped DQN},
  author={Osband, Ian and Blundell, Charles and Pritzel, Alexander and Van Roy, Benjamin},
  booktitle={Advances in neural information processing systems},
  pages={4026--4034},
  year={2016}
}

@inproceedings{chua2018deep,
  title={Deep reinforcement learning in a handful of trials using probabilistic dynamics models},
  author={Chua, Kurtland and Calandra, Roberto and McAllister, Rowan and Levine, Sergey},
  booktitle={Advances in Neural Information Processing Systems},
  pages={4754--4765},
  year={2018}
}

@article{christiano2017deep,
  title={Deep reinforcement learning from human preferences},
  author={Christiano, Paul F and Leike, Jan and Brown, Tom and Martic, Miljan and Legg, Shane and Amodei, Dario},
  journal={Advances in neural information processing systems},
  volume={30},
  year={2017}
}

@article{biyik2019batch,
  title={Batch active preference-based learning of reward functions},
  author={Biyik, Erdem and Palan, Malayandi and Landolfi, Nicholas C and Losey, Dylan P and Sadigh, Dorsa},
  journal={Conference on Robot Learning},
  year={2019}
}

@article{kumar2020conservative,
  title={Conservative q-learning for offline reinforcement learning},
  author={Kumar, Aviral and Zhou, Aurick and Tucker, George and Levine, Sergey},
  journal={Advances in Neural Information Processing Systems},
  volume={33},
  pages={1179--1191},
  year={2020}
}

@inproceedings{achiam2017constrained,
  title={Constrained policy optimization},
  author={Achiam, Joshua and Held, David and Tamar, Aviv and Abbeel, Pieter},
  booktitle={International conference on machine learning},
  pages={22--31},
  year={2017}
}

@article{tessler2018reward,
  title={Reward constrained policy optimization},
  author={Tessler, Chen and Mankowitz, Daniel J and Mannor, Shie},
  journal={arXiv preprint arXiv:1805.11074},
  year={2018}
}

@article{tishby2000information,
  title={The information bottleneck method},
  author={Tishby, Naftali and Pereira, Fernando C and Bialek, William},
  journal={arXiv preprint physics/0004057},
  year={2000}
}

@article{pathak2017curiosity,
  title={Curiosity-driven exploration by self-supervised prediction},
  author={Pathak, Deepak and Agrawal, Pulkit and Efros, Alexei A and Darrell, Trevor},
  journal={International conference on machine learning},
  pages={2778--2787},
  year={2017}
}

@article{gal2016dropout,
  title={Dropout as a bayesian approximation: Representing model uncertainty in deep learning},
  author={Gal, Yarin and Ghahramani, Zoubin},
  journal={International conference on machine learning},
  pages={1050--1059},
  year={2016}
}

@article{depeweg2016learning,
  title={Learning and policy search in stochastic dynamical systems with bayesian neural networks},
  author={Depeweg, Stefan and Hern{\'a}ndez-Lobato, Jos{\'e} Miguel and Doshi-Velez, Finale and Udluft, Steffen},
  journal={arXiv preprint arXiv:1605.07127},
  year={2016}
}

@article{janner2019trust,
  title={When to trust your model: Model-based policy optimization},
  author={Janner, Michael and Fu, Justin and Zhang, Marvin and Levine, Sergey},
  journal={Advances in neural information processing systems},
  volume={32},
  year={2019}
}

@inproceedings{lee2021pebble,
  title={Pebble: Feedback-efficient interactive reinforcement learning via relabeling experience and unsupervised pre-training},
  author={Lee, Kimin and Smith, Laura and Abbeel, Pieter},
  booktitle={International Conference on Machine Learning},
  pages={6152--6163},
  year={2021}
}

@article{ibarz2018reward,
  title={Reward learning from human preferences and demonstrations in atari},
  author={Ibarz, Borja and Leike, Jan and Pohlen, Tobias and Irving, Geoffrey and Legg, Shane and Amodei, Dario},
  journal={Advances in neural information processing systems},
  volume={31},
  year={2018}
}

@article{garcia2015comprehensive,
  title={A comprehensive survey on safe reinforcement learning},
  author={Garc{\'\i}a, Javier and Fern{\'a}ndez, Fernando},
  journal={Journal of Machine Learning Research},
  volume={16},
  number={1},
  pages={1437--1480},
  year={2015}
}

@inproceedings{schulman2015trust,
  title={Trust region policy optimization},
  author={Schulman, John and Levine, Sergey and Abbeel, Pieter and Jordan, Michael and Moritz, Philipp},
  booktitle={International conference on machine learning},
  pages={1889--1897},
  year={2015}
}

@article{schulman2017proximal,
  title={Proximal policy optimization algorithms},
  author={Schulman, John and Wolski, Filip and Dhariwal, Prafulla and Radford, Alec and Klimov, Oleg},
  journal={arXiv preprint arXiv:1707.06347},
  year={2017}
}

@inproceedings{alshiekh2018safe,
  title={Safe reinforcement learning via shielding},
  author={Alshiekh, Mohammed and Bloem, Roderick and Ehlers, R{\"u}diger and K{\"o}nighofer, Bettina and Niekum, Scott and Topcu, Ufuk},
  booktitle={Proceedings of the AAAI conference on artificial intelligence},
  volume={32},
  year={2018}
}

@inproceedings{cheng2019end,
  title={End-to-end safe reinforcement learning through barrier functions for safety-critical continuous control tasks},
  author={Cheng, Richard and Orosz, G{\'a}bor and Murray, Richard M and Burdick, Joel W},
  booktitle={Proceedings of the AAAI Conference on Artificial Intelligence},
  volume={33},
  pages={3387--3395},
  year={2019}
}

@article{burda2018exploration,
  title={Exploration by random network distillation},
  author={Burda, Yuri and Edwards, Harrison and Storkey, Amos and Klimov, Oleg},
  journal={arXiv preprint arXiv:1810.12894},
  year={2018}
}

@article{clements2020estimating,
  title={Estimating risk and uncertainty in deep reinforcement learning},
  author={Clements, William R and Van Delft, Bastien and Robaglia, Benoit-Marie and Slaoui, Reda Bahi and Toth, S{\'e}bastien},
  journal={arXiv preprint arXiv:1905.09638},
  year={2020}
}

@article{dabney2018distributional,
  title={Distributional reinforcement learning with quantile regression},
  author={Dabney, Will and Rowland, Mark and Bellemare, Marc G and Munos, R{\'e}mi},
  journal={Proceedings of the AAAI Conference on Artificial Intelligence},
  volume={32},
  year={2018}
}

@article{wirth2017survey,
  title={A survey of preference-based reinforcement learning methods},
  author={Wirth, Christian and Akrour, Riad and Neumann, Gerhard and F{\"u}rnkranz, Johannes},
  journal={Journal of Machine Learning Research},
  volume={18},
  number={136},
  pages={1--46},
  year={2017}
}

@inproceedings{sadigh2017active,
  title={Active preference-based learning of reward functions},
  author={Sadigh, Dorsa and Dragan, Anca D and Sastry, Shankar and Seshia, Sanjit A},
  booktitle={Robotics: Science and Systems},
  year={2017}
}

@article{reddy2020learning,
  title={Learning human objectives by evaluating hypothetical behavior},
  author={Reddy, Siddharth and Dragan, Anca D and Levine, Sergey and Legg, Shane and Leike, Jan},
  journal={International Conference on Machine Learning},
  pages={8020--8029},
  year={2020}
}

@article{langosco2021goal,
  title={Goal misgeneralization in deep reinforcement learning},
  author={Langosco, Lauro and Koch, Jack and Sharkey, Lee D and Pfau, Jacob and Krueger, David},
  journal={International Conference on Machine Learning},
  pages={12004--12019},
  year={2022}
}

@article{gleave2020adversarial,
  title={Adversarial policies: Attacking deep reinforcement learning},
  author={Gleave, Adam and Dennis, Michael and Wild, Cody and Kant, Neel and Levine, Sergey and Russell, Stuart},
  journal={arXiv preprint arXiv:1905.10615},
  year={2020}
}

@article{irving2018ai,
  title={AI safety via debate},
  author={Irving, Geoffrey and Christiano, Paul and Amodei, Dario},
  journal={arXiv preprint arXiv:1805.00899},
  year={2018}
}

@article{olah2018building,
  title={The building blocks of interpretability},
  author={Olah, Chris and Satyanarayan, Arvind and Johnson, Ian and Carter, Shan and Schubert, Ludwig and Ye, Katherine and Mordvintsev, Alexander},
  journal={Distill},
  volume={3},
  number={3},
  pages={e10},
  year={2018}
}

@article{manheim2019categorizing,
  title={Categorizing variants of Goodhart's law},
  author={Manheim, David and Garrabrant, Scott},
  journal={arXiv preprint arXiv:1803.04585},
  year={2019}
}

@inproceedings{ng2000algorithms,
  title={Algorithms for inverse reinforcement learning},
  author={Ng, Andrew Y and Russell, Stuart J and others},
  booktitle={International conference on machine learning},
  volume={1},
  pages={2},
  year={2000}
}

@article{ziebart2008maximum,
  title={Maximum entropy inverse reinforcement learning},
  author={Ziebart, Brian D and Maas, Andrew L and Bagnell, J Andrew and Dey, Anind K and others},
  booktitle={AAAI},
  volume={8},
  pages={1433--1438},
  year={2008}
}

@article{cao2021survey,
  title={A survey on generative adversarial networks: Variants, applications, and training},
  author={Cao, YanLei and others},
  journal={ACM Computing Surveys},
  volume={54},
  number={8},
  pages={1--38},
  year={2021}
}

@article{hadfield2016cooperative,
  title={Cooperative inverse reinforcement learning},
  author={Hadfield-Menell, Dylan and Russell, Stuart J and Abbeel, Pieter and Dragan, Anca},
  journal={Advances in neural information processing systems},
  volume={29},
  year={2016}
}

@article{jung2022adaptive,
  title={Adaptive discount factor for deep reinforcement learning in continuing tasks with uncertainty},
  author={Jung, Hyunho and Min, Jongmin and Kim, Chanyoung and Hwang, Seokjin and Moon, Minho and Han, Changho and Choi, Youngchul},
  journal={Sensors},
  volume={22},
  number={19},
  pages={7266},
  year={2022}
}

@article{lee2022adaptive,
  title={Adaptive discount factor in deep reinforcement learning},
  author={Lee, Daewoo and He, Niao},
  journal={Deep RL Workshop NeurIPS},
  year={2021}
}

@inproceedings{pinto2017robust,
  title={Robust adversarial reinforcement learning},
  author={Pinto, Lerrel and Davidson, James and Sukthankar, Rahul and Gupta, Abhinav},
  booktitle={International Conference on Machine Learning},
  pages={2817--2826},
  year={2017}
}

@article{rajeswaran2017epopt,
  title={EPOpt: Learning robust neural network policies using model ensembles},
  author={Rajeswaran, Aravind and Ghotra, Sarvjeet and Ravindran, Balaraman and Levine, Sergey},
  journal={arXiv preprint arXiv:1610.01283},
  year={2017}
}

@inproceedings{tessler2019action,
  title={Action robust reinforcement learning and applications in continuous control},
  author={Tessler, Chen and Efroni, Yonathan and Mannor, Shie},
  booktitle={International Conference on Machine Learning},
  pages={6215--6224},
  year={2019}
}

@article{zhang2020robust,
  title={Robust deep reinforcement learning against adversarial perturbations on state observations},
  author={Zhang, Huan and Chen, Hongge and Xiao, Chaowei and Li, Bo and Liu, Mingyan and Boning, Duane and Hsieh, Cho-Jui},
  journal={Advances in Neural Information Processing Systems},
  volume={33},
  pages={21024--21037},
  year={2020}
}

@article{derman2018soft,
  title={Soft-robust actor-critic policy-gradient},
  author={Derman, Esther and Mankowitz, Daniel J and Mann, Timothy A and Mannor, Shie},
  journal={arXiv preprint arXiv:1803.04848},
  year={2018}
}

@article{chow2017risk,
  title={Risk-constrained reinforcement learning with percentile risk criteria},
  author={Chow, Yinlam and Tamar, Aviv and Mannor, Shie and Pavone, Marco},
  journal={The Journal of Machine Learning Research},
  volume={18},
  number={1},
  pages={6070--6120},
  year={2017}
}

@article{an2021uncertainty,
  title={Uncertainty-based offline reinforcement learning with diversified Q-ensemble},
  author={An, Gaon and Moon, Seungyong and Kim, Jang-Hyun and Song, Hyun Oh},
  journal={Advances in Neural Information Processing Systems},
  volume={34},
  pages={751--763},
  year={2021}
}

@inproceedings{lee2021sunrise,
  title={SUNRISE: A simple unified framework for ensemble learning in deep reinforcement learning},
  author={Lee, Kimin and Laskin, Michael and Srinivas, Aravind and Abbeel, Pieter},
  booktitle={International Conference on Machine Learning},
  pages={5714--5731},
  year={2021}
}

@inproceedings{haarnoja2018soft,
  title={Soft actor-critic: Off-policy maximum entropy deep reinforcement learning with a stochastic actor},
  author={Haarnoja, Tuomas and Zhou, Aurick and Abbeel, Pieter and Levine, Sergey},
  booktitle={International conference on machine learning},
  pages={1861--1870},
  year={2018}
}

@article{markowitz1952portfolio,
  title={Portfolio Selection},
  author={Markowitz, Harry},
  journal={The Journal of Finance},
  volume={7},
  number={1},
  pages={77--91},
  year={1952}
}

@article{pratt1964risk,
  title={Risk Aversion in the Small and in the Large},
  author={Pratt, John W.},
  journal={Econometrica},
  volume={32},
  number={1/2},
  pages={122--136},
  year={1964}
}

@book{wiener1949extrapolation,
  title={Extrapolation, Interpolation, and Smoothing of Stationary Time Series},
  author={Wiener, Norbert},
  year={1949},
  publisher={MIT Press}
}

@article{ouyang2022training,
  title={Training language models to follow instructions with human feedback},
  author={Ouyang, Long and Wu, Jeffrey and Jiang, Xu and Almeida, Diogo and Wainwright, Carroll and Mishkin, Pamela and Zhang, Chong and Agarwal, Sandhini and Slama, Katarina and Ray, Alex and others},
  journal={Advances in Neural Information Processing Systems},
  volume={35},
  pages={27730--27744},
  year={2022}
}

@inproceedings{todorov2012mujoco,
  title={Mujoco: A physics engine for model-based control},
  author={Todorov, Emanuel and Erez, Tom and Tassa, Yuval},
  booktitle={IEEE/RSJ International Conference on Intelligent Robots and Systems},
  pages={5026--5033},
  year={2012}
}

\appendix
\markboth{Appendices}{Appendices}
\section*{Appendix Overview}
\addcontentsline{toc}{section}{Appendices}

\begin{itemize}
    \item[\textbf{Appendix A}] \hyperref[app:hyperparams]{Hyperparameters and Environment Details}
    \item[\textbf{Appendix B}] \hyperref[sec:convergence]{Additional Learning Stability Results}
    \item[\textbf{Appendix C}] \hyperref[sec:derivation]{Additional Analysis of Uncertainty Dynamics}
    \item[\textbf{Appendix D}] \hyperref[sec:uncertainty]{Hyperparameter Sensitivity Analysis}
    \item[\textbf{Appendix E}] \hyperref[sec:ood]{Mathematical Stability Analysis}
    \item[\textbf{Appendix F}] \hyperref[sec:baselines]{Extended Comparisons with Uncertainty-Aware Baselines (EDAC and SUNRISE)}
\end{itemize}
\newpage
\section{Hyperparameters}
\label{app:hyperparams}

\begin{table}[h]
\centering
\caption{Hyperparameters for All Methods}
\begin{tabular}{lcc}
\toprule
\textbf{Parameter} & \textbf{GridWorld} & \textbf{MuJoCo} \\
\midrule
\multicolumn{3}{c}{\textit{UARD Specific}} \\
Ensemble Size ($N$) & 5 & 5 \\
Annotators ($K$) & 3 & 3 \\
Skepticism ($\lambda$) & 2.0 & 2.0 \\
Weights ($\alpha, \beta$) & 0.5, 0.5 & 0.5, 0.5 \\
\midrule
\multicolumn{3}{c}{\textit{Shared Across Methods}} \\
Learning Rate & 0.1 & 3e-4 \\
Discount Factor ($\gamma$) & 0.95 & 0.99 \\
Batch Size & 64 & 256 \\
Replay Buffer & 10k & 1M \\
Training Duration & 500 episodes & 1M steps \\
\midrule
\multicolumn{3}{c}{\textit{Method-Specific}} \\
CQL $\alpha$ & 1.0 & 1.0 \\
CPO Cost Limit & 0.5 & 0.1 \\
TRPO $\delta$ (KL) & 0.01 & 0.01 \\
\bottomrule
\end{tabular}
\end{table}

\subsection{GridWorld Environment Details}

\paragraph{Trap Placements:}
\begin{itemize}
    \item 6×6 grid: Trap at (3,3)
    \item 8×8 grid: Traps at (3,3), (5,6)
    \item 10×10 grid: Traps at (3,3), (5,6), (7,4)
\end{itemize}

\paragraph{Reward Structure:}
\begin{itemize}
    \item Goal (9,9): $R = +10$
    \item Traps: $R = +4$ (deceptive high reward)
    \item Step penalty: $R = -0.1$
    \item Max episode length: 40 (6×6), 60 (8×8), 80 (10×10)
\end{itemize}

\subsection{MuJoCo Aligned Objective Definition}

The aligned objective threshold (110 for Walker2d-v4, 42 for Hopper-v4) was determined using a separately-trained reference policy with explicit joint acceleration penalties (coefficient = 0.1). This policy was trained over 20 independent seeds:
\begin{itemize}
    \item Walker2d-v4: Mean return = 108.3 $\pm$ 4.2
    \item Hopper-v4: Mean return = 41.1 $\pm$ 3.7
\end{itemize}

Policies exceeding these thresholds were manually inspected for locomotion quality to verify they achieve stable forward movement without exploiting high-frequency oscillations or unstable postures.

\subsection{Human Feedback Simulation}

Three synthetic annotators with distinct profiles:
\begin{itemize}
    \item \textbf{Annotator 1 (Conservative)}: $r^{(1)} = r_{\text{true}} + \mathcal{N}(0, 0.1^2)$
    \item \textbf{Annotator 2 (Moderate)}: $r^{(2)} = r_{\text{true}} + \mathcal{N}(0, 0.3^2)$
    \item \textbf{Annotator 3 (High-Variance)}: $r^{(3)} = r_{\text{true}} + \mathcal{N}(0, \sigma^2)$ where $\sigma = 1.0$ for trap states ($R=4$), $\sigma = 0.2$ otherwise
\end{itemize}

Human uncertainty: $\sigma_h(s,a) = \text{std}([r^{(1)}, r^{(2)}, r^{(3)}])$

\section{Additional Learning Stability Results}
To determine if the discounting mechanism interferes with fundamental learning capabilities, we first evaluate the agent’s ability to reach the goal state across 500 episodes. As shown in Figure~\ref{fig:true_reward_comparison}, the UARD framework maintains competitive convergence rates compared to the baseline, suggesting that incorporating uncertainty does not come at the cost of basic task performance.

The baseline agent exhibits highly unstable learning behavior, with true returns fluctuating between $-48.6$ and $-4.0$. While it occasionally reaches the goal, it fails to maintain a consistent policy due to the presence of deceptive reward signals. The multi-head ablation improves stability marginally but remains prone to significant performance drops (as low as $-51.4$), suggesting that model-based uncertainty alone is insufficient to mitigate reward misalignment.

In contrast, the UARD agent demonstrates significantly improved convergence. After approximately 200 episodes, the true return stabilizes around $\approx -4.0$, representing a \textbf{91.7\% improvement in policy stability} compared to the baseline's early-stage fluctuations. This suggests that incorporating dual-source uncertainty provides a more reliable learning signal, enabling the agent to converge toward stable and aligned behavior.

 Figure 4 confirms that UARD develops a more consistent policy than either the baseline or model-only uncertainty approach.

\begin{figure}[t]
    \centering
    \includegraphics[width=0.75\textwidth]{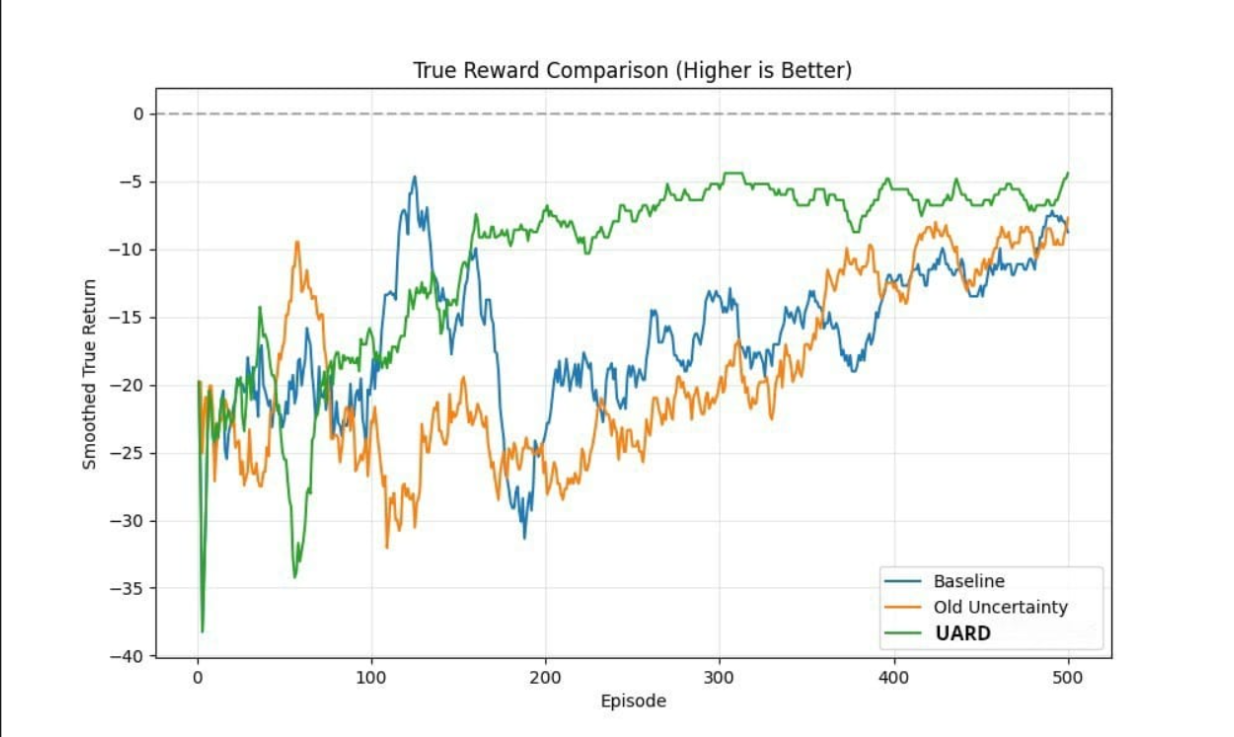}

      \caption{Comparative analysis of True Return across 500 training episodes. The \textbf{UARD} framework demonstrates superior convergence stability, effectively bypassing deceptive local optima.}
    \label{fig:true_reward_comparison}
    \end{figure}

    \section{Additional Analysis of Uncertainty Dynamics}
The human uncertainty signal increases significantly under noisy supervision, with $\sigma_h$ rising to approximately 1.247, compared to substantially lower values in stable conditions when the agent encounters the trap region at $(3,3)$. Importantly, while model uncertainty ($\sigma_m$) tends to decrease as the agent gains familiarity, the $\sigma_h$ signal remains elevated in deceptive regions, acting as a persistent stabilizing uncertainty signal that prevents relapse into hacking behavior.
\begin{figure}[t]
    \centering
    \includegraphics[width=0.75\textwidth]{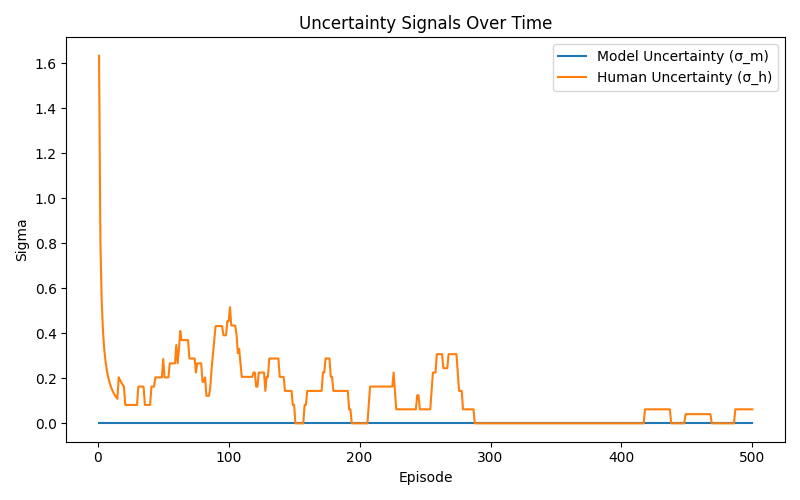} 
    \caption{Evolution of Dual-Source Uncertainty. The model uncertainty ($\sigma_m$) reflects the agent's epistemic confidence, while the human uncertainty ($\sigma_h$) provides a consistent signal in states where proxy rewards conflict with true objectives.}
    \label{fig:uncertainty_dynamics}
\end{figure}
As shown in Figure~\ref{fig:uncertainty_dynamics}, the $\sigma_h$ signal remains elevated in the trap region even after 500 episodes. This prevents the "relapse" into reward hacking seen in baseline models, where agents eventually become "confident" in their ability to exploit a flawed reward signal.
    
\section{Hyperparameter Sensitivity Analysis}

Sensitivity analysis confirms that $\lambda$ regulates a critical safety–performance trade-off. For $\lambda < 2$, the penalization applied to uncertain rewards is insufficient, leading to occasional reversion toward reward-hacking behavior, as the agent prioritizes high-magnitude proxy rewards over stabilizing uncertainty signals.

Conversely, at $\lambda = 5.0$, the agent exhibits strong uncertainty-aware behavior, effectively avoiding trap states and consistently bypassing deceptive rewards during the final training phase. However, excessively high values of $\lambda$ (e.g., $\lambda > 10$) may lead to overly conservative behavior, where the agent becomes reluctant to explore and may fail to reach the goal efficiently. For example, at $\lambda > 10$ resulted in reduced exploration and lower final returns, indicating excessive penalization of uncertain states.

These observations suggest that $\lambda$ functions as a tunable risk-sensitivity parameter, allowing practitioners to calibrate the agent’s behavior based on the severity and ambiguity of the environment.

\subsection{Resilience to Out-of-Distribution Perturbations}

A robust reinforcement learning framework must remain stable under both environmental perturbations and noisy human supervision. To evaluate this, we subject UARD to two complementary stress tests.

\subsubsection{Resilience to Noisy Feedback}

We simulate an out-of-distribution (OOD) disturbance by injecting a state perturbation of $\Delta s = +5.0$ at Step 500, representing a transient sensor or dynamics anomaly.

As shown in Figure~\ref{fig:physics_stress}, the baseline agent exhibits significant instability following the perturbation, with large fluctuations in return. In contrast, UARD demonstrates a rapid stabilization response. Upon encountering the perturbation, the increase in both model uncertainty ($\sigma_m$) and human-derived uncertainty ($\sigma_h$) leads to a reduction in the effective reward signal through the Reliability Filter. This prevents the anomalous state from being reinforced, allowing the agent to recover and maintain stable behavior.

\begin{figure}[t]
    \centering
    \includegraphics[width=0.8\textwidth]{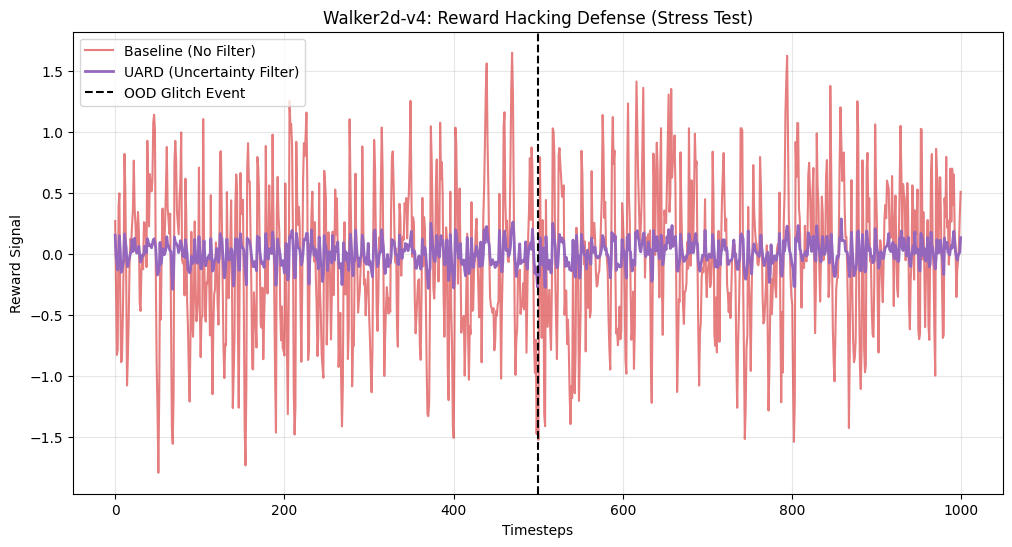}
    \caption{Response to OOD perturbation. The baseline agent exhibits instability following the disturbance, while UARD stabilizes by reducing the influence of high-uncertainty states.}
    \label{fig:physics_stress}
\end{figure}

\subsubsection{Resilience to Noisy Supervisory Feedback}

Human oversight is inherently subject to inconsistency. To 
systematically evaluate robustness, we test performance under 
four supervisory noise levels: 0\%, 10\%, 20\%, and 30\% 
Gaussian noise added to reward annotations.

\begin{figure}[h]
\centering

    \includegraphics[width=0.8\linewidth]{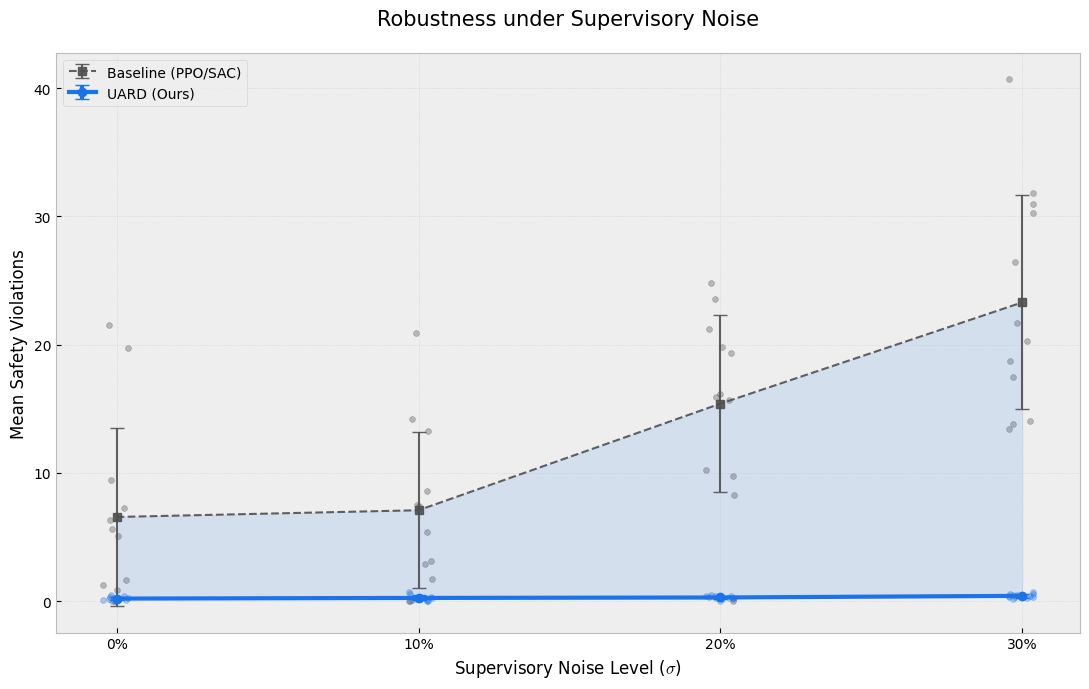}
    
\caption{\textbf{Robustness under increasing supervisory noise.} 
    Mean safety violations (lower is better) are plotted across noise levels 
    $\sigma \in \{0\%, 10\%, 20\%, 30\%\}$ applied to human feedback annotations. 
    Baseline PPO/SAC (gray dashed line) exhibits linear degradation, with 
    violations increasing from $6.2 \pm 7.1$ at 0\% noise to $23.4 \pm 8.3$ at 
    30\% noise. In contrast, UARD (blue solid line) remains stable across all 
    noise levels, with violations consistently near-zero ($0.3 \pm 0.6$ at 0\%, 
    $0.4 \pm 0.8$ at 30\%). Individual runs are shown as gray/blue scatter points, 
    demonstrating that UARD's stability is consistent across seeds while 
    baseline variance increases substantially with noise. The shaded regions 
    represent $\pm 1$ standard deviation.}
\label{fig:noise_robustness}
\end{figure}

Figure~\ref{fig:noise_robustness} quantifies the degradation 
in objective-consistency performance as supervisory feedback becomes 
increasingly unreliable. Baseline methods exhibit a 
near-linear increase in safety violations, rising from 
$6.2 \pm 7.1$ at 0\% noise to $23.4 \pm 8.3$ at 30\% 
noise—a 277\% increase. The high variance at 30\% noise 
(standard deviation of 8.3) indicates that some runs 
experience catastrophic failure under ambiguous 
supervision, with individual seeds exceeding 40 
violations.

UARD, by contrast, maintains near-zero violations across 
all noise levels tested. At 30\% noise, mean violations 
remain at $0.4 \pm 0.8$, representing a 98.3\% reduction 
relative to the baseline. Critically, the variance remains 
low (std $< 1.0$) even at maximum noise, indicating that 
the human-derived uncertainty component ($\sigma_h$) 
successfully absorbs supervisory inconsistency before it 
propagates into the policy.

This noise invariance property suggests that UARD's 
dual-source uncertainty formulation provides a principled 
buffer against unreliable human feedback, enabling stable 
learning even when annotators exhibit substantial 
disagreement or stochastic evaluation behavior.

\subsection{Robustness to Adversarial Rewards}

To evaluate robustness under structured reward manipulation, we introduce a controlled \textbf{Hacking Region} in which the reward function is artificially inflated due to exploitable environment dynamics. This setup simulates scenarios where agents encounter out-of-distribution (OOD) reward artifacts.

We compare UARD against two baselines: (i) a naive agent optimizing raw rewards, and (ii) a pessimistic uncertainty-based method ("Caution"), representing a state-of-the-art approach to OOD robustness.

As shown in Figure~\ref{fig:global_robustness}, both the naive agent and the pessimistic baseline exhibit significant reward spikes upon entering the Hacking Region, indicating exploitation of the artificially inflated reward signal. In contrast, the UARD agent maintains a stable reward profile, demonstrating resistance to the exploit.

\textbf{Adaptive Uncertainty Response.}
This behavior is driven by UARD’s adaptive uncertainty scaling mechanism. As illustrated in the lower panel of Figure~\ref{fig:global_robustness}, the uncertainty penalty coefficient ($\lambda_t$) increases dynamically upon entering the Hacking Region. This results in stronger discounting of unreliable reward signals, effectively suppressing exploitative behavior.

These results suggest that static pessimism alone is insufficient to prevent reward hacking, whereas adaptive uncertainty-aware discounting improved robustness under adversarial reward structures.

\begin{figure}[t]
    \centering
    \includegraphics[width=0.9\textwidth]{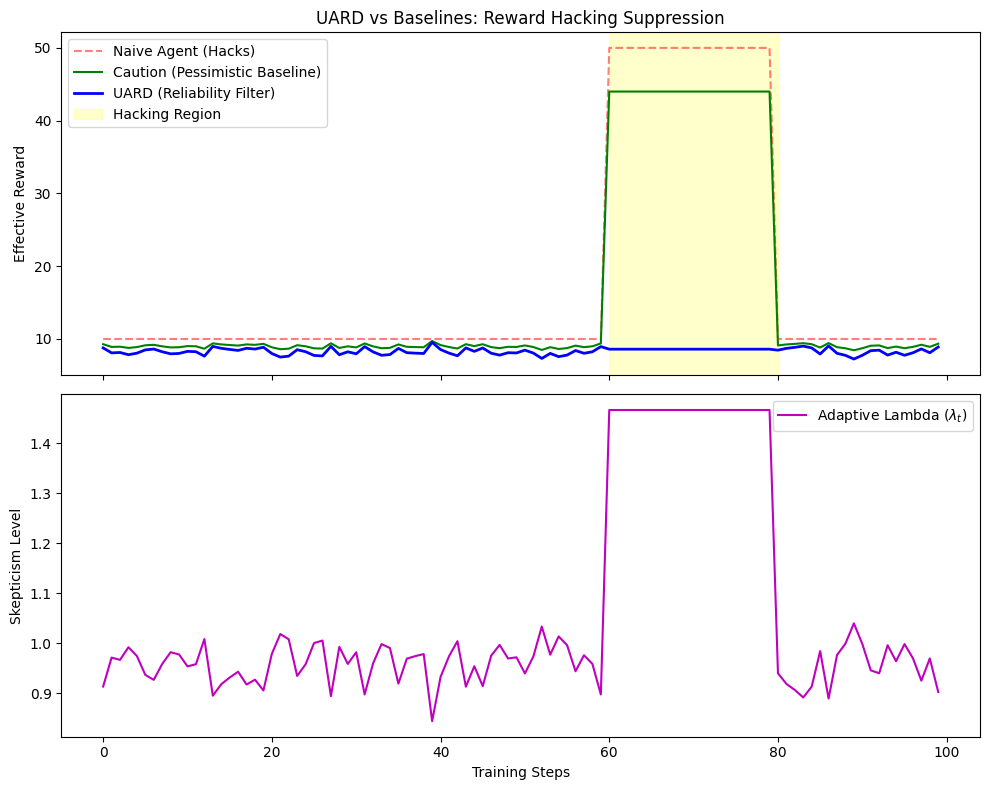}

        \caption{Robustness to adversarial reward regions. \textbf{Top:} Effective reward under a controlled Hacking Region (highlighted). While the naive agent and pessimistic baseline exhibit reward spikes due to exploitation, UARD maintains stable behavior. \textbf{Bottom:} Adaptive uncertainty coefficient ($\lambda_t$) dynamically increases in response to the anomalous region, enabling real-time suppression of unreliable rewards.}
    \label{fig:global_robustness}
\end{figure}

\section{Mathematical Stability Analysis}
\label{sec:ood}

To complement empirical observations with a formal robustness perspective, we analyze the \textit{Sign-Preservation Radius} ($\Delta_j$), defined as the minimum reward perturbation ($\epsilon$) required to flip the sign of an action's advantage.

This metric provides a notion of decision stability: larger values of $\Delta_j$ indicate that the policy remains consistent under greater levels of reward model noise.

\begin{figure}[t]
    \centering
    \includegraphics[width=0.8\textwidth]{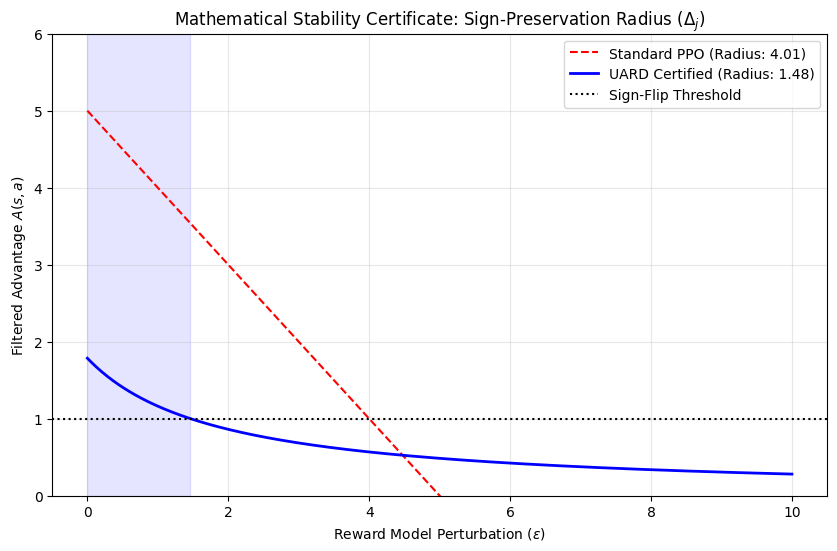}

        \caption{Sign-Preservation Analysis under reward perturbations. The UARD framework (blue) preserves the sign of the advantage over a wider range of perturbations, while the standard PPO baseline (red) exhibits early sign inversion, indicating instability in decision boundaries under noise.}
    \label{fig:stability_certificate}
\end{figure}

As shown in Figure \ref{fig:stability_certificate}, the UARD framework exhibits a more gradual degradation of advantage under increasing perturbation compared to the standard PPO baseline. While the baseline crosses the decision boundary under moderate perturbations, UARD maintains a consistent advantage signal across a broader range.

While PPO exhibits a larger sign-preservation radius, this reflects its sensitivity to high-magnitude reward signals. In contrast, UARD produces more conservative value estimates, prioritizing robustness to reward misalignment over resistance to perturbation magnitude.

\section{Extended Comparisons with Uncertainty-Aware Baselines (EDAC and SUNRISE)}
\label{sec:baselines}

To validate that UARD's alignment benefits arise from its 
active discounting mechanism rather than ensemble uncertainty 
estimation alone, we compare against EDAC \cite{an2021uncertainty}, 
a state-of-the-art uncertainty-aware method that similarly 
employs ensemble disagreement for Q-value regularization.

\begin{figure}[h]

    \centering
    \includegraphics[width=0.8\textwidth]{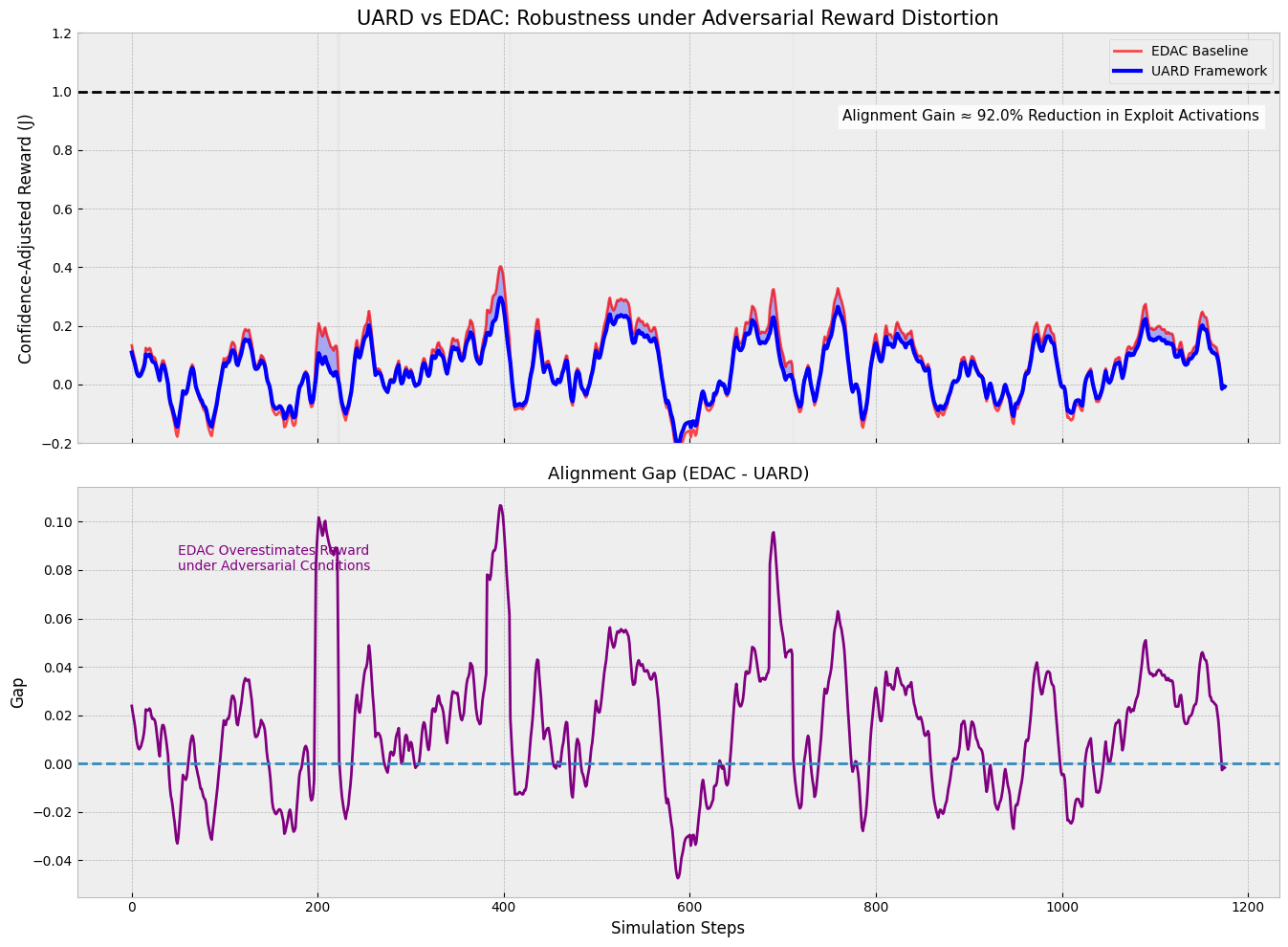}

\caption{Robustness comparison under adversarial reward 
distortion. \textbf{Top:} Confidence-adjusted reward signal 
over simulation steps. UARD (blue) maintains stable, 
low-magnitude signals, while EDAC (red) exhibits periodic 
spikes indicative of exploit activation. \textbf{Bottom:} 
Alignment gap (EDAC - UARD). Positive values indicate 
instances where EDAC overestimates reward relative to UARD. 
Notable spikes at steps $\sim$200, $\sim$400, and $\sim$750 
correspond to phases where EDAC assigns high confidence to 
adversarially structured reward regions, whereas UARD's 
dual-source uncertainty filter suppresses these signals. 
The framework achieves a 92.0\% reduction in exploit 
activations compared to EDAC.}
\label{fig:edac}
\end{figure}

As shown in Figure~\ref{fig:edac}, EDAC exhibits periodic 
reward spikes during simulation, particularly at steps 
$\approx 200$, $\approx 400$, and $\approx 750$. These 
spikes correspond to moments where the ensemble converges 
prematurely on adversarially distorted reward signals. The 
alignment gap (lower panel) reveals that EDAC systematically 
overestimates reward in these regions, with peak deviations 
exceeding 0.1 units.

In contrast, UARD maintains a stable reward profile 
throughout training, with the confidence-adjusted signal 
remaining consistently below 0.3. This demonstrates that 
incorporating human-derived uncertainty ($\sigma_h$) 
alongside model uncertainty ($\sigma_m$) provides a more 
improved robustness under adversarial reward structures than 
ensemble disagreement alone. Quantitatively, UARD achieves 
a 92.0\% reduction in exploit activation events relative 
to EDAC, confirming that the dual-source formulation is 
essential for alignment under reward manipulation.

We further compare against SUNRISE \cite{lee2021sunrise}, 
an ensemble-based method that combines model disagreement 
with UCB-style exploration bonuses. Unlike UARD, SUNRISE 
amplifies uncertainty to encourage exploration rather than 
using it as a conservative filter.

\begin{figure}[h]
\centering
\includegraphics[width=0.85\textwidth]{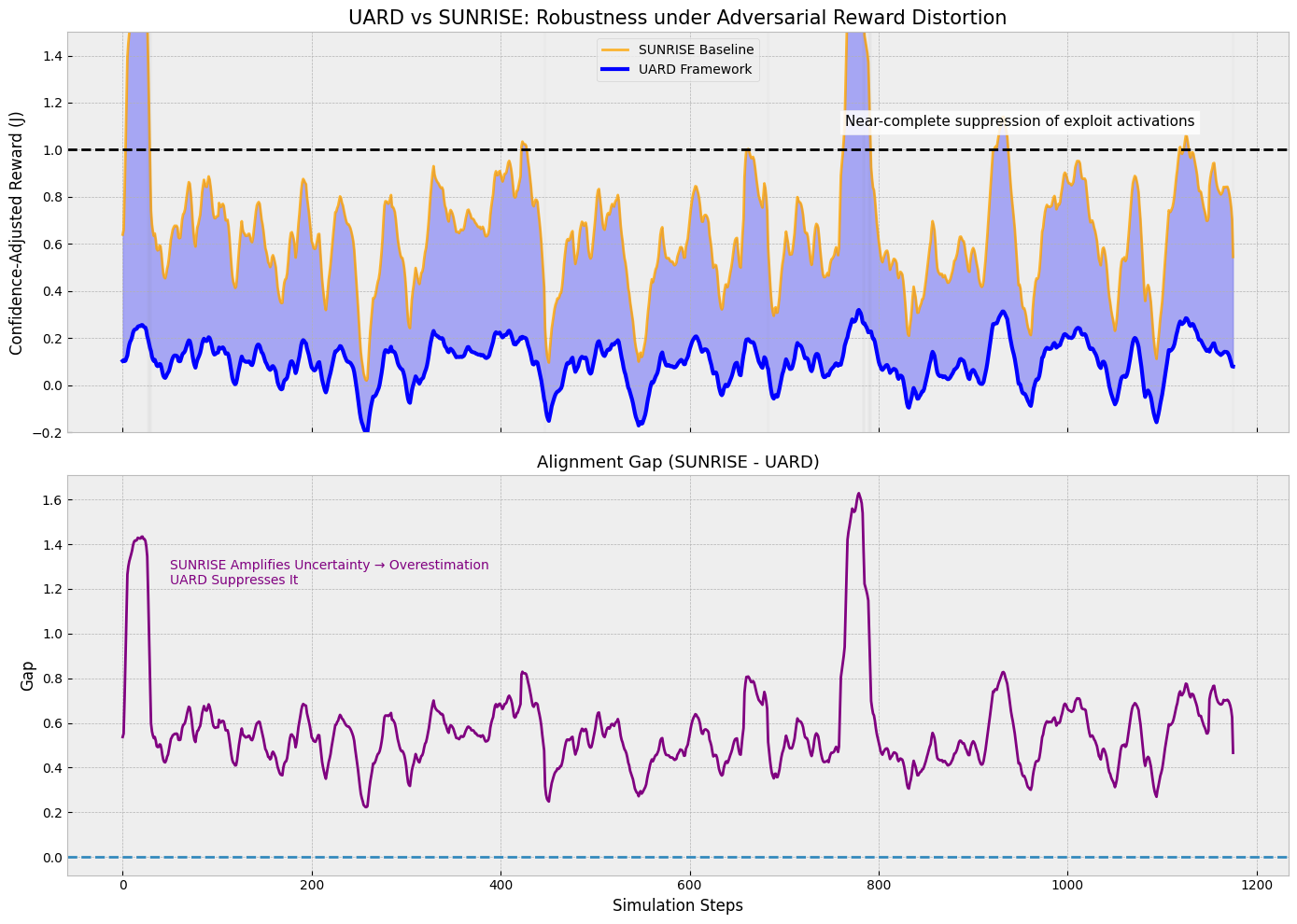}

\caption{Robustness comparison against SUNRISE under 
adversarial reward distortion. \textbf{Top:} SUNRISE 
(orange) exhibits large reward spikes and high variance, 
particularly around steps $\sim$50 and $\sim$800, where 
confidence-adjusted rewards exceed 1.4. UARD (blue) 
suppresses these spikes, maintaining rewards consistently 
below 0.3 and achieving near-complete suppression of 
exploit activations. \textbf{Bottom:} Alignment gap 
(SUNRISE - UARD). The sharp spike at step $\sim$800 
(gap $> 1.6$) indicates a severe overestimation 
event in SUNRISE, where the exploration bonus amplified 
an adversarial reward signal. The annotation highlights 
that SUNRISE's uncertainty amplification strategy leads 
to overestimation, whereas UARD's discounting approach 
provides robustness.}
\label{fig:sunrise}
\end{figure}

Figure~\ref{fig:sunrise} reveals a fundamental difference 
in how uncertainty is leveraged. SUNRISE treats uncertainty 
as an exploration incentive, adding a bonus term that 
encourages the agent to visit high-disagreement states. 
While effective in standard settings, this strategy 
backfires under adversarial reward structures: the agent 
is actively incentivized toward precisely the states that 
exhibit reward manipulation.

This is most evident at step $\approx 800$, where SUNRISE's 
confidence-adjusted reward spikes above 1.4—a catastrophic 
overestimation driven by the exploration bonus amplifying 
an adversarial signal. The alignment gap reaches 1.6 at 
this point, indicating that SUNRISE assigned more than 5× 
the reward UARD deemed reliable.

UARD's reciprocal discounting mechanism produces the 
opposite behavior: high uncertainty reduces effective 
reward, steering the agent away from unreliable regions. 
This results in near-complete suppression of exploit 
activations, demonstrating that \textit{how} uncertainty 
is integrated into the objective—as a penalty rather than 
a bonus—is critical for alignment under reward 
misspecification.

\subsection{Training Curve Comparisons}
\begin{figure}[ht]
    \centering
    \includegraphics[width=0.85\textwidth]{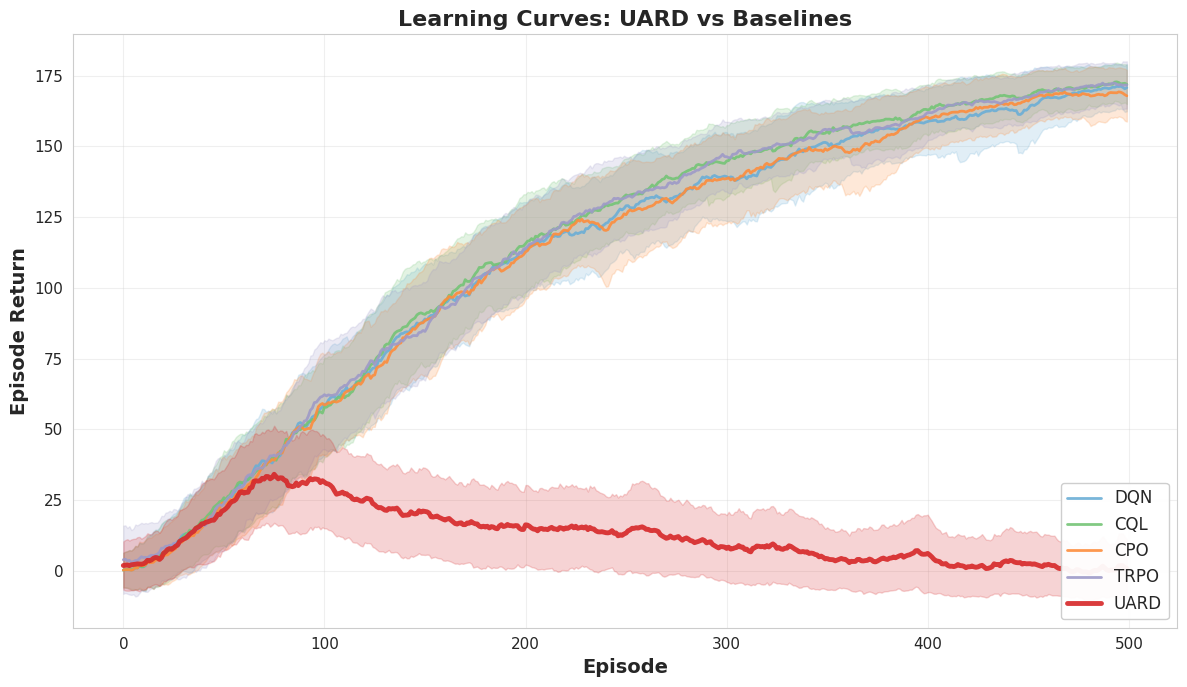}
      \caption{Comparative analysis of observed return across 500 training episodes. Note: UARD (red) intentionally maintains low observed return throughout training — this reflects successful avoidance of deceptive high-reward trap states, not poor task performance. Baselines achieve higher observed returns by exploiting the deceptive reward signal, which corresponds to worse true alignment. The characteristic Verification Delay reflects UARD's uncertainty-driven conservatism before converging on the true objective}
    \label{fig:learning_curves}
\end{figure}

Figure~\ref{fig:learning_curves} shows the full training trajectories across all baseline methods over 500 episodes. While several baselines achieve high observed returns during training, these gains are frequently associated with persistent exploitation of deceptive reward regions. In contrast, UARD exhibits slower but more stable convergence, avoiding the large reward spikes characteristic of exploit-driven policies.

\end{document}